\documentclass[10pt,journal,compsoc]{IEEEtran}
\ifCLASSOPTIONcompsoc
  \usepackage[nocompress]{cite}
\else
  \usepackage{cite}
\fi
\ifCLASSINFOpdf
\else
\fi
\usepackage{microtype}
\usepackage{graphicx}
\usepackage{booktabs} 

\usepackage{hyperref}
\usepackage[table,xcdraw]{xcolor}
\usepackage{subcaption}
\usepackage{graphicx}
\usepackage{amsmath}
\usepackage{amssymb}
\usepackage{booktabs}
\usepackage[table,xcdraw]{xcolor}
\usepackage{pifont}
\usepackage{listings}
\usepackage{colortbl}
\usepackage{xcolor} 
\usepackage{caption} 
\usepackage{listings}
\usepackage{multirow}
\definecolor{plainbackprop}{RGB}{180,0,0} 
\definecolor{sdobackprop}{RGB}{128,0,255} %
\usepackage[capitalize]{cleveref}
\usepackage{amsthm}
\usepackage{cancel}
\begin{document}
\newtheorem{definition}{\bf Definition}
\newtheorem{theorem}{\bf Theorem}[section]
\newtheorem{proposition}{\bf Proposition}[section]
\newtheorem{lemma}{\bf Lemma}
\newtheorem{remark}{\bf Remark}
%
\title{You Only Look One Step: Accelerating Backpropagation in Diffusion Sampling \\ with Gradient Shortcuts}
%
%
%
%

\author{Hongkun~Dou, Zeyu~Li, Xingyu~Jiang, Hongjue~Li, Lijun~Yang, Wen~Yao, and~Yue~Deng,~\IEEEmembership{Senior~Member,~IEEE}
\IEEEcompsocitemizethanks{
\IEEEcompsocthanksitem Hongkun Dou, Zeyu Li, Xingyu Jiang, Hongjue~Li and Lijun Yang are with the School of Astronautics, Beihang University, Beijing 100191, China. E-mail: douhk, lizeyu123478, jxy33zrhd, lihongjue, yanglijun@buaa.edu.cn.

\IEEEcompsocthanksitem  Wen Yao is with the Defense Innovation Institute, Chinese Academy of
Military Science, Beijing 100071, China. E-mail: wendy0782@126.com.

\IEEEcompsocthanksitem Yue Deng is with the Institute of Artificial Intelligence, Beihang University, Beijing 100191, China, and also with Beijing Zhongguancun Academy,
Beijing 100089, China. E-mail: ydeng@buaa.edu.cn.
}
}

%
%

\markboth{}
{Shell \MakeLowercase{\textit{et al.}}: Bare Demo of IEEEtran.cls for Computer Society Journals}
%



\IEEEtitleabstractindextext{%
\begin{abstract}
Diffusion models (DMs) have recently demonstrated remarkable success in modeling large-scale data distributions. However, many downstream tasks require guiding the generated content based on specific differentiable metrics, typically necessitating backpropagation during the generation process. This approach is computationally expensive, as generating with DMs often demands tens to hundreds of recursive network calls, resulting in high memory usage and significant time consumption. In this paper, we propose a more efficient alternative that approaches the problem from the perspective of parallel denoising. We show that full backpropagation throughout the entire generation process is unnecessary. The downstream metrics can be optimized by retaining the computational graph of only one step during generation, thus providing a shortcut for gradient propagation. The resulting method, which we call \textbf{Shortcut Diffusion Optimization (SDO)}, is generic, high-performance, and computationally lightweight, capable of optimizing all parameter types in diffusion sampling. We demonstrate the effectiveness of SDO on several real-world tasks, including controlling generation by optimizing latent and aligning the DMs by fine-tuning network parameters. Compared to full backpropagation, our approach reduces computational costs by $\sim 90\%$ while maintaining superior performance. Code is available at \href{https://github.com/deng-ai-lab/SDO}{https://github.com/deng-ai-lab/SDO}.
\end{abstract}

}

\maketitle

\IEEEdisplaynontitleabstractindextext

%
\IEEEpeerreviewmaketitle

\IEEEraisesectionheading{\section{Introduction}\label{sec:introduction}}

\begin{figure*}[!t]
      \centering 
\includegraphics[width=\textwidth]{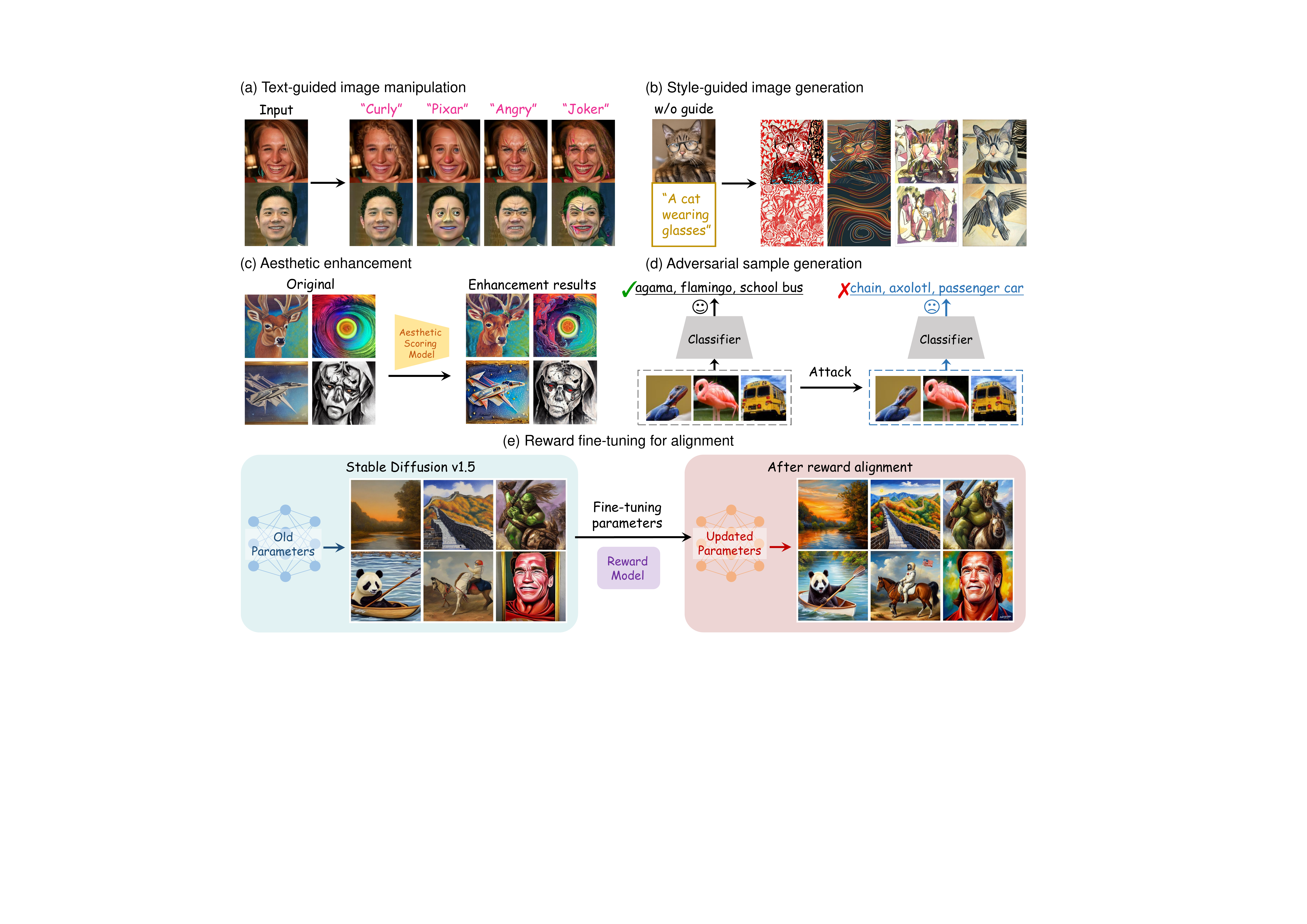}
\caption{\textbf{Shortcut Diffusion Optimization (SDO) enables diverse applications through backpropagation in diffusion sampling.} Parts (a)-(d) illustrate controlled generation by optimizing latent variables: (a) modifying an image based on text instructions using CLIP \cite{radford2021learning}; (b) controlling texture according to a given style reference while preserving content; (c) enhancing image aesthetics guided by an aesthetic predictor; and (d) generating adversarial perturbations that bypass well-trained classifiers. Part (e) demonstrates SDO’s capability for end-to-end fine-tuning of parameters, allowing the diffusion model to maximize specific rewards, such as aligning with human preferences.}\label{fig1}
\end{figure*}
%
%
%
%
\IEEEPARstart{D}{iffusion} models (DMs) \cite{ho2020denoising,song2020score} have successfully modeled large-scale data distributions and have become the standard for probabilistic modeling of continuous variables. Due to their flexibility and high performance, DMs are now widely applied in fields such as image synthesis \cite{rombach2022high,esser2024scaling,ramesh2022hierarchical, croitoru2023diffusion,bie2024renaissance}, video generation \cite{ho2022video,liu2024sora,yuan2025magictime}, molecular design \cite{xu2022geodiff,hoogeboom2022equivariant,luo2023fast} and control tasks \cite{janner2022planning,chi2023diffusion, hansen2023idql}. The core idea behind diffusion modeling is to learn the reverse process of noise generation, enabling the recovery of data from noise that approximates the target distribution. The learning objective typically involves approximately maximizing the likelihood. However, many applications of DMs focus on downstream tasks, such as improving the perceived quality of images \cite{wu2023human}, generating personalized styles \cite{sohn2023styledrop}, or optimizing attributes in drug design \cite{dorna2024tagmol}. During the pre-training phase, these objectives are difficult to achieve solely by maximizing the likelihood. Furthermore, training state-of-the-art DMs is computationally expensive, often requiring large-scale datasets and significant computing resources. As a result, a common approach is to customize pre-trained DMs for specific objectives. This allows the generation process to be controlled and adapted to new content without retraining the model from scratch \cite{chung2022diffusion,yu2023freedom}.

The common challenge in customizing and controlling diffusion models can be framed as an optimization problem, where the goal is to ensure that the generated content meets specific, desired characteristics. Specifically, we assume that this characteristic can be captured by a differentiable metric, allowing us to minimize the objective function to ensure that the sampling results align with the desired properties:
\begin{equation}\label{objective}
    \min_{\{\boldsymbol{x}_1,\epsilon_{\theta},\rho\}} \mathcal{J}(G(\boldsymbol{x}_1,\epsilon_{\theta},\rho))
\end{equation}
where $G$ represents the sampling process of the DM, and $\epsilon_{\theta}$ is the parameterized denoising network. The variables to be optimized are extensive, potentially including the model parameters $\theta$, the noisy latent variable $\boldsymbol{x}_T$, and the prompt embedding $ \rho $. For instance, one might optimize the latent to generate images with a style similar to a reference image \cite{yu2023freedom} or fine-tune the model’s parameters to improve the aesthetic score of the generated images \cite{prabhudesai2023aligning}.

A straightforward solution to this optimization problem is to perform gradient backpropagation throughout the diffusion sampling process. However, diffusion sampling involves an iterative denoising process, which requires making tens to hundreds of recursive network calls \cite{song2020denoising}. This approach leads to significant memory consumption. To address this, various solutions have been proposed, including the use of gradient checkpointing \cite{chen2016training}, adjoint methods \cite{panadjointdpm}, and reversible network modules \cite{wallace2023end}. However, these methods often come with a tradeoff in terms of computational time, as calculating gradients becomes more time-consuming. Additionally, recursive calculations can introduce issues such as gradient explosion or vanishing, posing further challenges for optimization.  This raises an important question: \textbf{is time-consuming full backpropagation truly necessary for diffusion models?}

In this paper, we demonstrate that surprisingly, it is possible to bypass much of the denoising process and efficiently optimize Eq.~\ref{objective} by retaining only a single step to compute the gradient. Our approach leverages recent advances in parallel denoising for DMs, which maintain full denoising trajectories and apply Picard iterative refinement until convergence at a fixed point \cite{shih2024parallel, luparasolver, tang2024accelerating}. We utilize insights from parallel denoising to reorganize the sequential sampling process in Eq.~\ref{objective}. By carefully managing gradient computation within the iterative process, we discover that it is sufficient to preserve the computational graph for only the denoising step of the variable being optimized. This provides a shortcut to backpropagation, dramatically reducing the computational burden. Our method, termed \textbf{Shortcut Diffusion Optimization (SDO)}, is a general and computationally lightweight algorithm that can optimize all parameter types in diffusion sampling. 

We validate the proposed SDO algorithm across several real-world tasks, including optimizing noisy latent for text-guided image manipulation, style-guided generation, aesthetic quality enhancement and adversarial sample generation. Additionally, we experiment with end-to-end fine-tuning of pre-trained DM parameters to align the model with specific reward (Fig.~\ref{fig1}). Despite its reduced computational complexity, we find that SDO frequently outperforms full backpropagation, delivering superior efficiency and performance across a wide range of tasks.

Our contributions are summarized as follows:
\begin{itemize}
    \item We address a fundamental yet underexplored problem: how to perform efficient backpropagation through the diffusion sampling chain. We provide a novel perspective by revisiting this issue through the lens of parallel sampling dynamics.

    \item Motivated by recent developments in the differentiation of fixed-point systems, we propose SDO, a principled and lightweight gradient approximation method. SDO retains only a single step in the computation graph, significantly reducing memory and computation costs.

    \item We empirically demonstrate that SDO achieves substantial efficiency gains over full backpropagation in downstream tasks such as controlled generation and reward alignment. Its stability, simplicity, and broad applicability make SDO a promising solution for scaling diffusion models to practical applications.
\end{itemize}

The remainder of this paper is organized as follows. In Section~\ref{relatedwork}, we review existing work relevant to Shortcut Diffusion Optimization (SDO). In Section~\ref{pre}, we provide background on diffusion models and introduce the foundation for parallel denoising. Section~\ref{meth} presents our core methodology: we begin by analyzing the limitations of plain backpropagation in diffusion sampling, then introduce the SDO framework, along with theoretical justifications and comparisons to existing alternatives. Section~\ref{app} demonstrates the effectiveness of SDO across a range of controllable generation and reward alignment tasks. Finally, Section~\ref{conc} concludes the paper and discusses potential directions for future research.

\section{Related Work} \label{relatedwork}

\subsection{Controllable Generation with diffusion models}
 Controllable generation using diffusion models (DMs) has been an active research area. Some approaches incorporate conditioning directly into the denoising networks \cite{zhang2023adding,mou2024t2i,xia2024diffusion,he2025diffusion}, while others utilize noise-aware classifiers or regressors to guide the sampling process \cite{dhariwal2021diffusion, nichol2021glide, weiss2023guided}. However, these methods often require task-specific training or paired data, which limits their flexibility. Another direction leverages DMs as plug-and-play priors. For instance, Weiss \emph{et al.} \cite{weiss2023guided} formulated an optimization framework that utilizes denoising loss as a regularization term for conditional sampling, although this often resulted in artifacts that required additional post-processing. Liu \emph{et al.} \cite{liu2023flowgrad} applied optimal control to achieve controllable generation by optimizing a set of control vectors. DITTO \cite{novack2024ditto} is a framework that can achieve the target output for text-to-music DMs by inference-time optimization. Other works, such as FreeDoM \cite{yu2023freedom} and UGD \cite{bansal2023universal}, proposed a one-step approximation using clean images to simplify the guidance term. While these approximations reduce computational complexity, they tend to be overly simplistic, leading to a degradation in the quality of generated images.

\subsection{Backpropagation through Diffusion Sampling}
Recently, Wallace \emph{et al.} \cite{wallace2023end} introduced DOODL, which enhances the guidance of DMs by enabling end-to-end optimization of latent variables. To perform backpropagation with constant memory during diffusion sampling, they employed invertible coupling layers \cite{dinh2014nice,dinh2016density}, addressing the challenge of memory-intensive gradient computations. Additionally, Pan \emph{et al.} \cite{panadjointdpm} proposed integrating adjoint sensitivity techniques \cite{pontryagin2018mathematical,chen2018neural} into diffusion ODEs to enable gradient computation across all parameter types in DMs. Blasingame \emph{et al.} \cite{blasingame2024adjointdeis} extended the adjoint sensitivity approach to diffusion SDEs. Although these methods reduce memory consumption, backpropagation through the entire sampling process remains computationally expensive, and the memory demands are still higher compared to standard DM training. In contrast, our work shows that full backpropagation is often unnecessary. We propose an efficient shortcut for gradient computation, presenting a significant step towards more practical backpropagation for diffusion sampling.

\subsection{Alignment of Text-to-Image Diffusion Models}
 Fine-tuning via reinforcement learning (RL) has been extensively used in language models to align model behavior with human preferences \cite{stiennon2020learning,ziegler2019fine,ouyang2022training}, and this approach has recently been extended to diffusion models. Lee \emph{et al.} \cite{lee2023aligning} and Wu  \emph{et al.} \cite{wu2023humanpreferencescorebetter} proposed using rewards to filter and weight samples to supervise the fine-tuning of diffusion models. More recently, Fan  \emph{et al.} \cite{fan2024reinforcement} and Black \emph{et al.} \cite{black2023training} formulated diffusion sampling as a Markov decision process, utilizing policy gradient-based RL algorithms to align model outputs with desired outcomes. Furthermore, Clark \emph{et al.} \cite{clark2023directly} and  Prabhudesai \emph{et al.} \cite{prabhudesai2023aligning} propose a faster fine-tuning approach than RL by directly backpropagating through differentiable rewards during the sampling process, reducing computational consumption via truncated backpropagation. Wu \emph{et al.} \cite{wu2024deep} proposed a heuristic to explore stopping the input gradient to the denoising network when fine-tuning diffusion models. Our proposed algorithm is compatible with this alignment paradigm. Unlike previous truncation methods, SDO introduces gradient shortcuts that eliminate unnecessary computations, resulting in a more efficient and stable solution.

\section{Preliminaries}\label{pre}
\subsection{Diffusion Models and Probability Flows}
The diffusion model (DM) \cite{ho2020denoising,song2020score} defines a stochastic differential equation (SDE) to perturb samples $\boldsymbol{x}_0 \sim p_{\text{data}}$ over time:
\begin{equation}
    \mathrm d\boldsymbol{x}_t=f(t)\boldsymbol{x}_t\mathrm dt+g(t)\mathrm d\boldsymbol{w}_t,
\end{equation}
where $\boldsymbol{w}_t$ is the standard Wiener process, $f(t)$ and $g(t)$ are the drift and diffusion coefficients, respectively. This forward process has the transition kernel $q(\boldsymbol{x}_t|\boldsymbol{x}_0) = \mathcal{N}(\alpha_t\boldsymbol{x}_0,{\sigma_t}^2\boldsymbol{I})$. By running this SDE, noise is gradually introduced into the initial sample $\boldsymbol{x}_0$, until it converges to a standard normal distribution $\boldsymbol{x}_1 \sim \mathcal{N}(\mathbf{0}, \boldsymbol{I})$ at time $t=1$. To recover the samples from the noise, we solve the corresponding inverse SDE step-by-step:
\begin{equation}
    \mathrm d\boldsymbol{x}_t=\left[f(t)\boldsymbol{x}_t-g^2(t)\nabla_{\boldsymbol{x}_t}\log q_{t}(\boldsymbol{x}_t)\right]\mathrm dt+g(t)\mathrm d\boldsymbol{w}_t,
\end{equation}
with the standard Wiener process $\bar{\boldsymbol{w}}_{t}$ running in reverse time.  The term $\nabla_{\boldsymbol{x}_t}\log q_{t}(\boldsymbol{x}_t)$, known as the score function, can be approximated as $\nabla_{\boldsymbol{x}_t}\log q_{t}(\boldsymbol{x}_t) \approx -\epsilon_{\theta}(\boldsymbol{x}_t,t)/\sigma_t$ by training a neural network using denoising score matching \cite{song2019generative}. The denoising process then runs the inverse SDE to recover the original samples from the noise. 

Alternatively, the same marginal densities of the above SDE can also be modeled through a probabilistic flow ordinary differential equation (PF-ODE) \cite{song2020score}:
\begin{equation}\label{eq_ode}
    \mathrm d\boldsymbol{x}_t=\underbrace{\left[f(t)\boldsymbol{x}_t-\frac{g^2(t)}{2\sigma_t}\epsilon_{\theta}(\boldsymbol{x}_t,t)\right]}_{\mathrm{velocity~field~}u_{\theta}(\boldsymbol{x}_t,t)}\mathrm dt,
\end{equation}
The ODE formulation has been shown to sample more efficiently than its SDE counterpart. Recent research suggests learning the velocity field directly through Flow Matching \cite{lipman2022flow,liu2022flow,esser2024scaling} to model these dynamics. Notably, our proposed SDO remains fully compatible with these broader models.

We can leverage off-the-shelf numerical solvers to solve Eq.~\ref{eq_ode} numerically. One widely used approach is the DDIM solver \cite{song2020denoising}, expressed as:
\begin{equation}\label{ddim}
\boldsymbol{x}_{n}=\boldsymbol{x}_{n+1}-\frac{1}{N} u_\theta\left(\boldsymbol{x}_{n+1}, \frac{n+1}{N}\right),
\end{equation}
where the integration interval $\left[0,1\right]$ is discretized with a step size of $1/N$. In practice, selecting an appropriate number of steps $N$  is crucial, with typical values ranging from tens to hundreds for PF-ODE. However, such a sequential sampling process results in a very deep computational graph, which poses significant challenges for backpropagation.

\begin{figure}
		\centering 
		\includegraphics[width=\columnwidth]{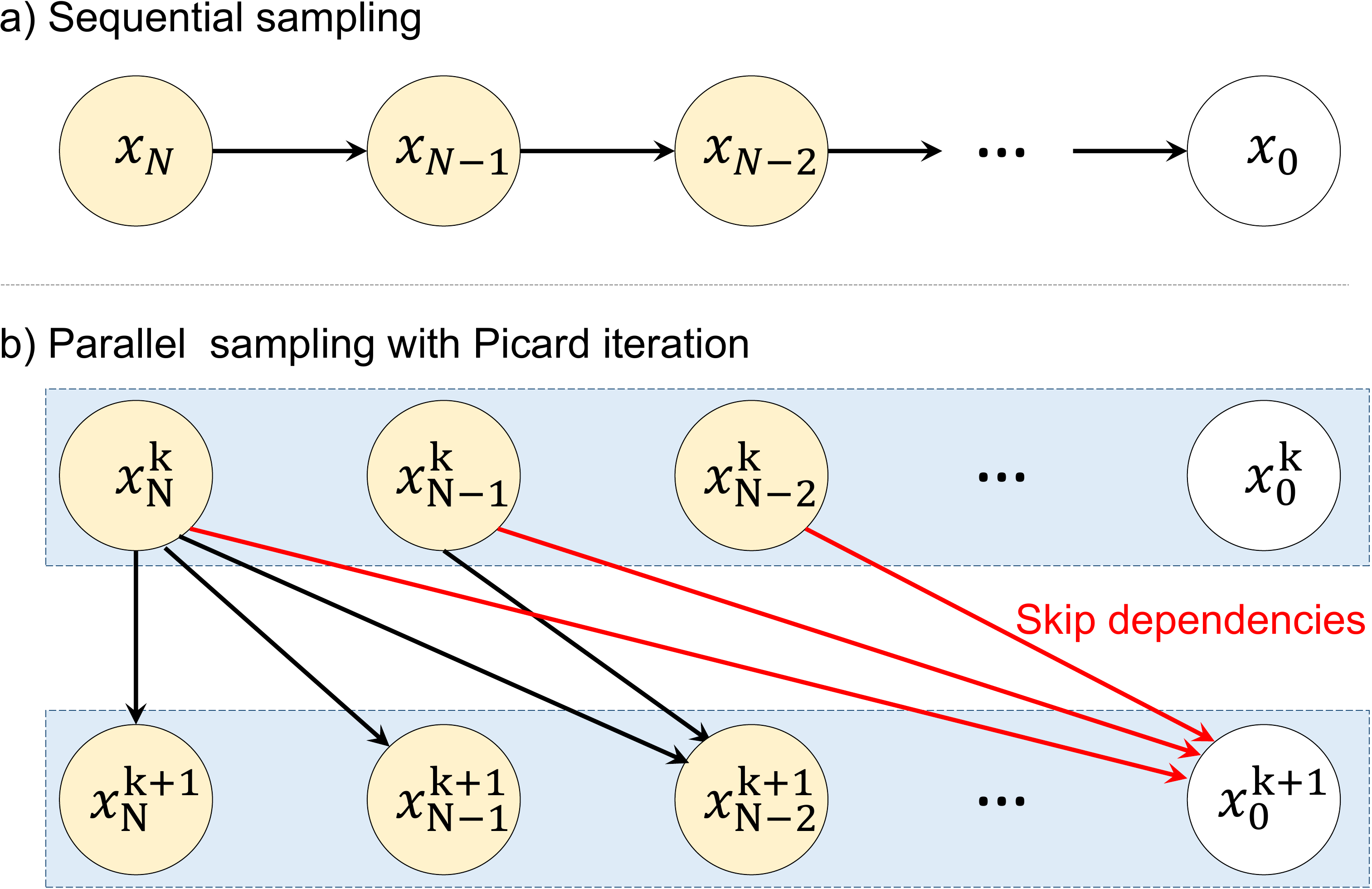}
		\caption{Graphical models of DMs illustrating (a) sequential sampling \cite{song2020denoising} and (b) parallel sampling using Picard iteration \cite{shih2024parallel}. Picard iteration introduces \textcolor{red}{skip dependencies} between $\boldsymbol{x}_0$ and $\{\boldsymbol{x}_N,\cdots,\boldsymbol{x}_1\}$, which inspired the development of SDO.}  \label{graph_2}
\end{figure}

\subsection{Parallel Denoising with Picard Iteration}
To enable parallel denoising across all time steps, Shih \emph{et al.} \cite{shih2024parallel} propose using \emph{Picard iteration}, a fixed-point method for solving ODEs. The core idea is to rewrite Eq.~\ref{ddim} in its integral form:
\begin{equation}\label{int}
    \boldsymbol{x}_n=\boldsymbol{x}_N+\int_1^{n/N}u_{\theta}(\boldsymbol{x}_t,t)\mathrm d t \approx \boldsymbol{x}_N-\frac{1}{N}\sum_{i=N}^{n+1} u_{\theta}(\boldsymbol{x}_i, \frac{i}{N})
\end{equation}
where the integral is approximated by a discrete sum over the time steps. The above equation allows the sampling process to be treated as a sequence of updates, which is iteratively refined until convergence. The iterative process begins with an initial guess for the sequence $\{\boldsymbol{x}^{k}_0,\boldsymbol{x}^{k}_2,\cdots, \boldsymbol{x}^{k}_{N}\}$ at initial iteration  $k=0$. The sequence is then updated using the following iterative rule:
\begin{equation}\label{int}
    \boldsymbol{x}^{k+1}_n= \boldsymbol{x}^{k}_N-\frac{1}{N}\sum_{i=N}^{n+1} u_{\theta}(\boldsymbol{x}^{k}_i, \frac{i}{N})
\end{equation}
Each timestep $n$ can be updated independently within a given Picard iteration, enabling parallel computation across all steps.  The update rule for the entire sequence can be expressed compactly as:
\begin{equation}\label{para}
\begin{aligned}
\underbrace{\left[\begin{array}{c}
        \boldsymbol{x}_0^{k+1}\\
        \boldsymbol{x}_1^{k+1}\\
        \vdots \\
\boldsymbol{x}_{N-1}^{k+1} \\
\boldsymbol{x}_{N}^{k+1} \\
\end{array}\right]}_{\boldsymbol{x}_{0:N}^{k+1}}=\underbrace{\left[\begin{array}{c}
\boldsymbol{x}_{N}^k-\frac{1}{N}\sum_{i=N}^{1} u_\theta\left(\boldsymbol{x}_{i}^k, \frac{i}{N}\right)\\
\boldsymbol{x}_{N}^k-\frac{1}{N}\sum_{i=N}^{2} u_\theta\left(\boldsymbol{x}_{i}^k, \frac{i}{N}\right)\\
\vdots \\
\boldsymbol{x}_{N}^k-\frac{1}{N} u_\theta\left(\boldsymbol{x}_{N}^k, 1\right) \\
\boldsymbol{x}_{N}^k
\end{array}\right]}_{\mathcal{F}_{\theta}(\boldsymbol{x}_{0:N}^{k})}\\
\end{aligned}
\end{equation}
In this formulation, each element in the sequence at iteration $k+1$ depends on the corresponding element and all previous elements from the previous iteration $k$. This dependency effectively introduces \emph{skip dependencies} (shown in Fig.~\ref{graph_2}) into the computational graph, improving the propagation of information in the sampling chain. For a sufficiently large number of iterations, we have,
\begin{equation}\label{fix-cond}
    \lim_{k \rightarrow \infty} \mathcal{F}_{\theta}(\boldsymbol{x}_{0:N}^{k}) = \mathcal{F}_{\theta}(\boldsymbol{x}_{0:N}^{*}) = \boldsymbol{x}_{0:N}^{*}
\end{equation}
Empirically, the number of Picard iterations required to reach convergence is significantly smaller than the total number of timesteps $N$. 

\begin{figure*}
		\centering 
		\includegraphics[width=\textwidth]{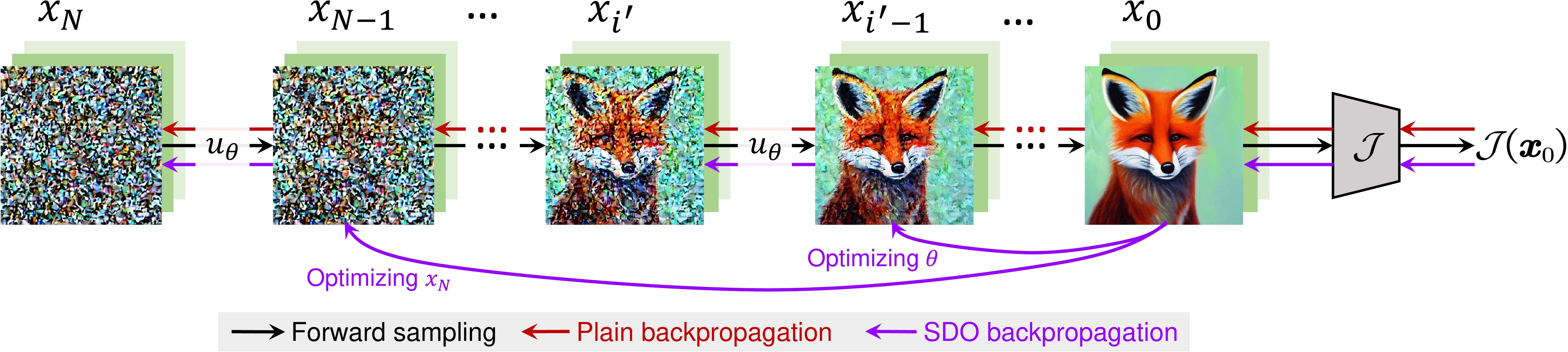}
		\caption{\textcolor{plainbackprop}{Plain backpropagation} inefficiently differentiates the objective function through the entire forward sampling process. In contrast, \textcolor{sdobackprop}{SDO backpropagation} requires differentiation only at the specific timestep associated with the variable being optimized (e.g., $\boldsymbol{x}_N$ at step $N$ or $\theta$ at step $i^\prime$), providing an efficient \textcolor{sdobackprop}{shortcut} for the backward pass of the gradient.} \label{sample-figure}
\end{figure*}

\section{Method} \label{meth}
In this section, we first derive the gradient for plain backpropagation in DMs and analyze the key contributors to its memory inefficiency (Sec.~\ref{4.1}). To address these challenges, we reformulate the optimization problem using the Picard iteration-induced fixed-point system. We demonstrate how this approach avoids specific memory bottlenecks, offering a more efficient shortcut for gradient propagation (Sec.~\ref{4.2}).

\subsection{Analysis of Plain Backpropagation} \label{4.1}
In plain backpropagation of diffusion sampling, gradient computation involves backpropagating through the entire generation process, which can be represented as a series of recursive updates resembling a recurrent neural network (RNN) \cite{medsker2001recurrent,hochreiter1997long}. For a given objective $\mathcal{J}(\boldsymbol{x}_0)$, the optimization problem is expressed as:
\begin{equation}\label{prob1}
\begin{aligned}
        &\min_{\{\boldsymbol{x}_N,\theta\}} \mathcal{J}(\boldsymbol{x}_0)\\
        s.t.\quad &\boldsymbol{x}_{n}=\boldsymbol{x}_{n+1}-\frac{1}{N} u_\theta\left(\boldsymbol{x}_{n+1}, \frac{n+1}{N}\right)\\ &
         \forall  \ n\in \{0,1,\cdots,N\}.
\end{aligned}
\end{equation}
For simplicity, we do not explicitly consider prompt embeddings, as they can be treated as network parameters. Optimizing such an RNN-like computational graph typically involves Backpropagation Through Time (BPTT) \cite{rumelhart1986learning,werbos1990backpropagation}.  The gradient formulas for the latent variable $\boldsymbol{x}_N$ and the parameter $\theta$ are given by:
\begin{subequations}
\begin{align}
    &\nabla_{\boldsymbol{x}_N}\mathcal{J}(\boldsymbol{x}_0)= \frac{\partial \mathcal{J}}{\partial \boldsymbol{x}_0} \textcolor{orange}{\left(\prod_{j=1}^{N-1} \frac{\partial \boldsymbol{x}_{j-1}}{\partial \boldsymbol{x}_{j}}\right)} \frac{\partial \boldsymbol{x}_{N-1}}{\partial \boldsymbol{x}_{N}}\label{grad_a}\\
    &\nabla_{\theta}\mathcal{J}(\boldsymbol{x}_0)=-\frac{1}{N}\sum_{i=1}^{N} \frac{\partial \mathcal{J}}{\partial \boldsymbol{x}_0}\textcolor{orange}{\left(\prod_{j=1}^{i} \frac{\partial \boldsymbol{x}_{j-1}}{\partial \boldsymbol{x}_{j}}\right)} \frac{\partial u_\theta\left(\boldsymbol{x}_{i},\frac{i}{N}\right)}{\partial \theta}\label{grad_b}
\end{align}
\end{subequations}
The primary source of computational complexity lies in the product terms $\prod_{j} \frac{\partial \boldsymbol{x}_{j-1}}{\partial \boldsymbol{x}_{j}}$ (highlighted in \textcolor{orange}{orange}). These terms require storing all intermediate activations during the generation process, leading to high memory consumption. Although techniques such as gradient checkpointing, adjoint sensitivity methods, and reversible networks can reduce memory consumption, they often come with additional computational overhead, limiting overall efficiency.

\subsection{Shortcut Diffusion Optimization (SDO)} \label{4.2}

To address the memory inefficiency inherent in plain backpropagation for DMs, we propose \textbf{Shortcut Diffusion Optimization (SDO)}. SDO leverages Picard iteration to reformulate Eq.~\ref{prob1} as an optimization problem for a fixed-point system. Specifically, after a sufficient number of iterations, once convergence is achieved, the solution must satisfy the fixed-point condition (Eq.~\ref{fix-cond}). Thus, the optimization problem can be expressed as:
\begin{equation}\label{prob2}
\begin{aligned}
        &\min_{\{\boldsymbol{x}_N,\theta\}} \mathcal{J}(\boldsymbol{x}_{0:N}^{*})\\
s.t.\quad&\boldsymbol{x}_{0:N}^{*}=\mathcal{F}_\theta(\boldsymbol{x}_{0:N}^{*}).\\
\end{aligned}
\end{equation}
Our key insight is that the fixed-point constraints in the above equation naturally align with the original sequential sampling chain. In other words, solving Eq.~\ref{prob2} is equivalent to solving Eq.~\ref{prob1}. Consequently, we analyze the gradient of the fixed-point system to develop an efficient optimization strategy. 

\subsubsection{Efficient Gradient Computation}
The differentiation of iterative dynamics has been well-explored in the literature, with traditional methods relying on costly implicit differentiation techniques \cite{lorraine2020optimizing, blondel2022efficient}. Let $(\cdot)$ denote the variable for which the gradient is computed. The sequence $\boldsymbol{x}_{0:N}^{*}$ can be interpreted as an implicitly defined function of $(\cdot)$. By applying the implicit function theorem (IFT) \cite{griewank2008evaluating, krantz2002implicit} to the constraint in Eq.~\ref{prob2}, the gradient with respect to $(\cdot)$ can be computed as follows:

\begin{equation}\label{ift}
    \nabla_{(\cdot)}\mathcal{J}(\boldsymbol{x}_{0:N}^{*}) =  \frac{\partial \mathcal{J}}{\partial \boldsymbol{x}_{0:N}^{*}} \underbrace{\left( I- \frac{\partial \mathcal{F}_\theta(\boldsymbol{x}_{0:N}^{*})}{\partial \boldsymbol{x}_{0:N}^{*}}\right)^{-1}}_{\text{matrix
inversion}}  \frac{\partial \mathcal{F}_\theta(\boldsymbol{x}_{0:N}^{*})}{\partial (\cdot)} 
\end{equation}
In this expression, implicit differentiation involves solving the inversion of a linear system, which can be computationally prohibitive, especially for large-scale models. Nevertheless, recent advances \cite{geng2021attentionbettermatrixdecomposition, fung2022jfb, bolte2024one, du2022learning} have demonstrated that better results can be achieved by using approximate, inexpensive gradients, a technique known as one-step gradient. This can be formally expressed as:
\begin{equation}\label{impli}
    \widehat{\nabla_{(\cdot)}\mathcal{J}(\boldsymbol{x}_{0:N}^{*})} =  \frac{\partial \mathcal{J}}{\partial \boldsymbol{x}_{0:N}^{*}}  \frac{\partial \mathcal{F}_\theta(\boldsymbol{x}_{0:N}^{*})}{\partial (\cdot)} = \frac{\partial \mathcal{J}}{\partial \boldsymbol{x}_{0:N}}  \frac{\partial \mathcal{F}_\theta(\boldsymbol{x}_{0:N})}{\partial (\cdot)}
\end{equation}
The one-step gradient can be interpreted as differentiating only for the final iteration, bypassing the need for matrix inversion and offering a more efficient alternative to Eq.~\ref{ift}. The second equality of Eq.~\ref{impli} arises because the sequence $\boldsymbol{x}_{0:N}^{*}$ still satisfies the constraint of DDIM discretization in Eq.~\ref{ddim}, as indicated by the conclusion of the following proposition.

\begin{proposition} Given that $\boldsymbol{x}_{0:N}^{*}$ is the fixed point of Eq.~\ref{fix-cond} and $\boldsymbol{x}_{0:N}$ is the DDIM discretization as defined in Eq.~\ref{ddim}, if the initial noise $\boldsymbol{x}_{N}^{*}= \boldsymbol{x}_{N}^{*}$, then we have $\boldsymbol{x}_{0:N}^{*} = \boldsymbol{x}_{0:N}$. \label{prop1} \end{proposition}

\begin{proof}
Since the initial noise satisfies $x_N = x^*_N$, we can derive the following:
\begin{equation}
    \boldsymbol{x}_{N-1} = \boldsymbol{x}_N - \frac{1}{N} u_\theta(\boldsymbol{x}_N, 1) = \boldsymbol{x}^*_N - \frac{1}{N} u_\theta(\boldsymbol{x}^*_N, 1) = \boldsymbol{x}^*_{N-1} \nonumber
\end{equation}
Next, using the above relationship, we further deduce:
\begin{equation}
    \begin{aligned}
        \boldsymbol{x}_{N-2} &=  \boldsymbol{x}_{N-1} - \frac{1}{N} u_\theta(\boldsymbol{x}_{N-1}, \frac{N-1}{N})\\
        &= \boldsymbol{x}_{N-1}^* - \frac{1}{N} u_\theta(\boldsymbol{x}_{N-1}^*, \frac{N-1}{N})\\
        &=\boldsymbol{x}_{N}^*-\frac{1}{N}\sum_{i=N}^{N-1}  u_\theta(\boldsymbol{x}_{i}^*, \frac{i}{N})=\boldsymbol{x}_{N-2}^* \nonumber
    \end{aligned}
\end{equation}
By simple mathematical induction, it follows that  $\boldsymbol{x}_{i} = \boldsymbol{x}_{i}^*$ for all $i \in [0,1,...,N]$. Hence, the conclusion is valid.
\end{proof}

For simplicity, we can omit the superscript $*$ in the subsequent derivations.

\subsubsection{Gradient w.r.t. the Latent Variable $\boldsymbol{x}_N$}
We first consider the optimization of the latent variable $\boldsymbol{x}_N$. Since the objective function $\mathcal{J}$ typically depends only on the endpoint sample $\boldsymbol{x}_0$, the gradient of Eq.~\ref{impli} can be simplified as:
\begin{equation}\label{g_x}
\begin{aligned}
    \widehat{\nabla_{\boldsymbol{x}_{N}}\mathcal{J}(\boldsymbol{x}_{0})}&=\frac{\partial \mathcal{J}}{\partial \boldsymbol{x}_{0}}\frac{\partial \left(\boldsymbol{x}_{N}-\frac{1}{N}\sum_{i=N}^{1} u_\theta\left(\boldsymbol{x}_{i}, \frac{i}{N}\right)\right)}{\partial \boldsymbol{x}_{N}}\\&=\frac{\partial \mathcal{J}}{\partial \boldsymbol{x}_{0}}  \frac{\partial \left(\boldsymbol{x}_{N}-\frac{1}{N} u_\theta\left(\boldsymbol{x}_{N}, 1\right)\right)}{\partial \boldsymbol{x}_{N}}\\&=\frac{\partial \mathcal{J}}{\partial \boldsymbol{x}_{0}}  \frac{\partial \boldsymbol{x}_{N-1}}{\partial \boldsymbol{x}_{N}}
    \end{aligned}
\end{equation}
Compared to the original gradient in Eq.~\ref{grad_a}, this simplified formulation shows that we bypass the product terms of the Jacobian highlighted in \textcolor{orange}{orange}, effectively avoiding the need to backpropagate through the entire sequence of intermediate steps. Instead, we compute the gradient by focusing only on the relationship between the final two states, which significantly reduces computational overhead.

\subsubsection{Gradient w.r.t. Network Parameters $\theta$}
 Similarly, the gradient with respect to the network parameters $\theta$ is given by:
  \begin{equation}\label{SDO_theta_eq}
     \begin{aligned}
    \widehat{\nabla_{\theta}\mathcal{J}(\boldsymbol{x}_{0})}&=\frac{\partial \mathcal{J}}{\partial \boldsymbol{x}_{0}}\frac{\partial \left(\boldsymbol{x}_{N}-\frac{1}{N}\sum_{i=N}^{1} u_\theta\left(\boldsymbol{x}_{i}, \frac{i}{N}\right)\right)}{\partial \theta}\\&=\frac{\partial \mathcal{J}}{\partial \boldsymbol{x}_{0}}\frac{\partial \left(-\frac{1}{N}\sum_{i=N}^{1} u_\theta\left(\boldsymbol{x}_{i}, \frac{i}{N}\right)\right)}{\partial \theta}\\& 
  =-\frac{1}{N}\sum_{i=1}^{N}\frac{\partial \mathcal{J}}{\partial \boldsymbol{x}_{0}}  \frac{\partial u_\theta\left(\boldsymbol{x}_{i},\frac{i}{N}\right)}{\partial \theta}
    \end{aligned}
 \end{equation}
From the above derivation, we observe a similar conclusion: the burdensome product term in Eq.~\ref{grad_b} can also be eliminated when optimizing the network parameters. However, the summation over all denoising steps still implies a need to store all intermediate states, which would increase memory requirements. To address this issue, we reinterpret Eq.~\ref{SDO_theta_eq} as an expectation, which can be approximated by sampling a single timestep $i^\prime$ from a uniform distribution over all timesteps $i\sim \mathrm{Uniform}(\{1,\cdots,N\})$ at each backpropagation phase:
\begin{equation}\label{g_theta}
     \begin{aligned}
   \widehat{\nabla_{\theta}\mathcal{J}(\boldsymbol{x}_{0:N})}&= -\mathbb{E}_{i}\left[\frac{\partial \mathcal{J}}{\partial \boldsymbol{x}_{0}}  \frac{\partial u_\theta\left(\boldsymbol{x}_{i},\frac{i}{N}\right)}{\partial \theta}\right]\\&\approx -\frac{\partial \mathcal{J}}{\partial \boldsymbol{x}_{0}}  \frac{\partial u_\theta\left(\boldsymbol{x}_{i^\prime},\frac{i^\prime}{N}\right)}{\partial \theta}
   \end{aligned}
\end{equation}
This approximation allows us to avoid storing all intermediate states, reducing memory consumption while still effectively optimizing the network parameters.

\subsection{Practical Implementation of SDO}
In practice, implementing the proposed gradient computation is straightforward. While our derivation relies on Picard parallel sampling, rather than the commonly used sequential diffusion solver, Proposition~\ref{prop1} demonstrates the equivalence of the two denoising trajectories. As such, SDO is fully compatible with standard off-the-shelf diffusion solvers, and the sampling process in the forward pass of the model remains unchanged. The key difference lies in the backpropagation: we only retain the computational graph for the specific timestep associated with the optimization variables. For all other steps, we use \texttt{set\_grad\_enabled} in PyTorch \cite{paszke2019pytorch} for $u_{\theta}$ to skip unnecessary Jacobian computations during backpropagation.  We show the PyTorch-style implementation of our SDO in Figs. \ref{SDO_x} and \ref{SDO_theta}. 

Note that \texttt{torch.set\_grad\_enabled} disables gradients only for network inference, i.e., when calculating $u_\theta\left(\cdot,\cdot\right)$. Importantly, the call to \texttt{schedule.step} does not disable gradients, as it is responsible for computing $\boldsymbol{x}_{n-1}$ from $\boldsymbol{x}_{n}$. Consequently, the sampling result in Fig. \ref{SDO_x} can be written as:
\begin{equation}
    \boldsymbol{x}_{0} = \boldsymbol{x}_{N}- \frac{1}{N}u_\theta(\boldsymbol{x}_N,1) - \frac{1}{N}\sum_{i=N-1}^{1}\text{sg}(u_\theta(\boldsymbol{x}_i,\frac{i}{N}))
\end{equation}
where $\text{sg}(\cdot)$ denotes the operation where gradients are disabled. By differentiating this expression, we obtain:
\begin{equation}
\small
    \begin{aligned}
        \frac{\partial \mathcal{J}}{\partial \boldsymbol{x}_{N} } &= \frac{\partial \mathcal{J}}{\partial \boldsymbol{x}_{0} }\frac{\partial\boldsymbol{x}_{0} }{\partial \boldsymbol{x}_{N} }\\
        &=\frac{\partial \mathcal{J}}{\partial \boldsymbol{x}_{0} }\frac{\partial \left(\boldsymbol{x}_{N}- \frac{1}{N}u_\theta(\boldsymbol{x}_N,1) - \cancel{\frac{1}{N}\sum_{i=N-1}^{1}\text{sg}(u_\theta(\boldsymbol{x}_i,\frac{i}{N}))}\right) }{\partial \boldsymbol{x}_{N} }\\
        &=\frac{\partial \mathcal{J}}{\partial \boldsymbol{x}_{0} }\frac{\partial\boldsymbol{x}_{N-1} }{\partial \boldsymbol{x}_{N} }
    \end{aligned}
\end{equation}
This result is consistent with Eq.~\ref{g_x} in the prior section, which allows for obtaining the gradient of $\boldsymbol{x}_{N}$ for the proposed SDO. This approach can also be easily extended to compute gradients for arbitrary intermediate states $\boldsymbol{x}_i$, where $i \in [1,2,...,N]$, by retaining the network computation graphs for the corresponding time steps.

Similarly, assuming $i^\prime$ is the timestep that was sampled, and that the gradient for all other timesteps is stopped, the same reasoning applies to Fig.~\ref{SDO_theta}. Specifically: 
\begin{equation}
\small
    \begin{aligned}
        \frac{\partial \mathcal{J}}{\partial \theta } &= \frac{\partial \mathcal{J}}{\partial \boldsymbol{x}_{0} }\frac{\partial\boldsymbol{x}_{0} }{\partial \theta }\\
        &=\frac{\partial \mathcal{J}}{\partial \boldsymbol{x}_{0} }\frac{\partial \left(\cancel{\boldsymbol{x}_{N}- \frac{1}{N}\sum_{i\neq i^\prime} \text{sg}(u_\theta(\boldsymbol{x}_i,\frac{i}{N}))}- \frac{1}{N}u_\theta(\boldsymbol{x}_{i^\prime},\frac{i^\prime}{N})\right) }{\partial \theta }\\
        &=-\frac{1}{N}\frac{\partial \mathcal{J}}{\partial \boldsymbol{x}_{0} }\frac{\partial u_\theta(\boldsymbol{x}_{i^\prime},\frac{i^\prime}{N}) }{\partial \theta}
    \end{aligned}
\end{equation}
This conclusion is proportional to Eq.~\ref{g_theta}, differing by only a constant factor. Thus, this formulation can be used to compute the gradient of SDO for the network parameters, and it also applies to other shared variables across timesteps, such as prompt embeddings.

Fig.~\ref{sample-figure} provides a visualization of how SDO operates. Intuitively, due to the skip dependencies of Picard iteration, SDO offers a shortcut for gradient computation in the sampling chain, effectively reducing both the computational and memory overhead.
Interestingly, although SDO discards multiple steps of Jacobian computation, our experimental results revealed an unexpected benefit: in many cases, SDO outperforms full backpropagation not only in speed but also in accuracy. This aligns with recent empirical observations of one-step gradients in prior works \cite{du2022learning, geng2021attentionbettermatrixdecomposition, fung2022jfb}.  We hypothesize that this improvement is due to the gradient shortcuts mitigating the problem of gradient explosion, leading to faster and more stable convergence (see Fig.~\ref{fig:sub1}).

\begin{figure}[h!]
    \centering
    \begin{lstlisting}
params = {'params': x_N, 'lr': lr}
optimizer = torch.optim.Adam([params])

for _ in range(epochs):  
    optimizer.zero_grad()
    x_t = x_N
    for i, t in enumerate(scheduler.timesteps):
        with torch.set_grad_enabled(i == 0):
            noise_pred = model(x_t, t)
        x_t = scheduler.step(noise_pred, t, x_t)
    output = torch.clamp(x_t, -1, 1)
    loss = J(output)
    loss.backward()
    optimizer.step()
    \end{lstlisting}
    \caption{\textbf{PyTorch implementation of SDO for optimizing $\boldsymbol{x}_N$.} The function \texttt{set\_grad\_enabled} ensures that only the one-step computation graph required for backpropagation is retained. \texttt{scheduler} stands for solver scheduling, e.g. DDIM \cite{song2020denoising}.}
    \label{SDO_x}
\end{figure}

\begin{figure}[h!]
    \centering
    \begin{lstlisting}
params = {'params': model.parameters(), 'lr': lr}
optimizer = torch.optim.Adam([params])

for _ in range(epochs):  
    optimizer.zero_grad()
    x_t = x_N
    backprop_step = random.randint(0, 
                        len(scheduler.timesteps))
    for i, t in enumerate(scheduler.timesteps):
        is_grad = (i == backprop_step)
        with torch.set_grad_enabled(is_grad):
            noise_pred = model(x_t, t)
        x_t = scheduler.step(noise_pred, t, x_t)
    output = torch.clamp(x_t, -1, 1)
    loss = J(output)
    loss.backward()
    optimizer.step()
    \end{lstlisting}
    \caption{\textbf{PyTorch implementation of SDO for optimizing model parameters.} The program samples the timesteps and retains the computational graph only for the sampled timesteps, ensuring lightweight backpropagation while enabling optimization across all timesteps.}
    \label{SDO_theta}
\end{figure}

\subsection{Theoretical Analysis of SDO}
The gradient shortcuts in SDO rely on a one-step gradient approximation for the iterative dynamics. A natural question is how accurate this approximate gradient is compared to the full raw gradient. In this section, we provide the theoretical justification for this approximation, showing that under some mild conditions, the one-step gradient used by SDO is bounded and provides a principled approximation.

\begin{theorem}
Suppose the metric $\mathcal{J}(\cdot)$ has gradients that are $\rho$-bounded. Additionally, assume that the update function $\mathcal{F}_\theta (\cdot)$ is $L_{\mathcal{F}}$-Lipschitz with respect to $\theta$ and is a contraction with constant $\lambda \in [0, 1)$. Then, the following bounds hold:
\begin{subequations}
\begin{align}
     &\| \nabla_{\boldsymbol{x}_N}\mathcal{J}(\boldsymbol{x}_{0:N})- \widehat{\nabla_{\boldsymbol{x}_N}\mathcal{J}(\boldsymbol{x}_{0:N})}\| \leq \frac{\lambda^2\rho}{1-\lambda} \nonumber\\
     &\| \nabla_{\theta}\mathcal{J}(\boldsymbol{x}_{0:N})- \widehat{\nabla_{\theta}\mathcal{J}(\boldsymbol{x}_{0:N})}\| \leq \frac{\lambda \rho L_{\mathcal{F}}}{1-\lambda} \nonumber
 \end{align}
 \end{subequations}
\end{theorem}
\begin{proof}
    First,  since $\mathcal{F}_\theta$ is a contraction with constant $\lambda \in [0,1)$, we know that the operator norm $\|\frac{\partial \mathcal{F}_\theta}{ \partial \boldsymbol{x}_{0:N}}\| \leq \lambda$, implying that for any vector $v$, we have $\|(I-\frac{\partial \mathcal{F}_\theta}{ \partial \boldsymbol{x}_{0:N}})v\|\geq(1-\lambda )\|v\|$.  Since $(1-\lambda )\|v\|> 0$, for any non-zero vector $v$, it follows that $I-\frac{\partial \mathcal{F}_\theta}{ \partial \boldsymbol{x}_{0:N}}$ is invertible. 

Then, by the fact $((I-\frac{\partial \mathcal{F}_\theta}{ \partial \boldsymbol{x}_{0:N}})^{-1}-I)(I-\frac{\partial \mathcal{F}_\theta}{ \partial \boldsymbol{x}_{0:N}})=\frac{\partial \mathcal{F}_\theta}{ \partial \boldsymbol{x}_{0:N}}$, we obtain $(I-\frac{\partial \mathcal{F}_\theta}{ \partial \boldsymbol{x}_{0:N}})^{-1}-I = \frac{\partial \mathcal{F}_\theta}{ \partial \boldsymbol{x}_{0:N}}(I-\frac{\partial \mathcal{F}_\theta}{ \partial \boldsymbol{x}_{0:N}})^{-1}$.

Using this result, we can derive:
\begin{equation}
    \small
    \begin{aligned}
         &\| \nabla_{(\cdot)}\mathcal{J}(\boldsymbol{x}_{0:N})- \widehat{\nabla_{(\cdot)}\mathcal{J}(\boldsymbol{x}_{0:N})}\| \\
         &=\| \tfrac{\partial \mathcal{J}}{\partial \boldsymbol{x}_{0:N}} \left( I- \tfrac{\partial \mathcal{F}_\theta}{\partial \boldsymbol{x}_{0:N}}\right)^{-1} \tfrac{\partial \mathcal{F}_\theta}{\partial (\cdot)}  - \tfrac{\partial \mathcal{J}}{\partial \boldsymbol{x}_{0:N}}\tfrac{\partial \mathcal{F}_\theta}{\partial (\cdot)}\|\\
           &= \| \tfrac{\partial \mathcal{J}}{\partial \boldsymbol{x}_{0:N}} \left(\left( I- \tfrac{\partial \mathcal{F}_\theta}{\partial \boldsymbol{x}_{0:N}}\right)^{-1} -I\right)\tfrac{\partial \mathcal{F}_\theta}{\partial (\cdot)}\|\\
           &= \| \tfrac{\partial \mathcal{J}}{\partial \boldsymbol{x}_{0:N}} \left(\tfrac{\partial \mathcal{F}_\theta}{ \partial \boldsymbol{x}_{0:N}}(I-\tfrac{\partial \mathcal{F}_\theta}{ \partial \boldsymbol{x}_{0:N}})^{-1}\right)\tfrac{\partial \mathcal{F}_\theta}{\partial (\cdot)}\|\\
           &\leq \|\tfrac{\partial \mathcal{J}}{\partial \boldsymbol{x}_{0:N}} \| \|\tfrac{\partial \mathcal{F}_\theta}{ \partial \boldsymbol{x}_{0:N}}\| (1-\|\tfrac{\partial \mathcal{F}_\theta}{ \partial \boldsymbol{x}_{0:N}}\|)^{-1}\|\tfrac{\partial \mathcal{F}_\theta}{ \partial (\cdot)}\|\nonumber
    \end{aligned}
\end{equation}
Since $\|\tfrac{\partial \mathcal{J}}{\partial \boldsymbol{x}_{0:N}} \|\leq \rho$, $\|\tfrac{\partial \mathcal{F}_\theta}{ \partial \boldsymbol{x}_{0:N}}\|\leq \lambda$, and $\|\tfrac{\partial \mathcal{F}_\theta}{ \partial \theta}\| \leq L_{\mathcal{F}}$, we conclude: 
\begin{subequations}
  \small
\begin{align}
     \| \nabla_{\boldsymbol{x}_N}\mathcal{J}(\boldsymbol{x}_{0:N})- \widehat{\nabla_{\boldsymbol{x}_N}\mathcal{J}(\boldsymbol{x}_{0:N})}\| \leq \frac{\lambda^2\rho}{1-\lambda} \nonumber\\
     \| \nabla_{\theta}\mathcal{J}(\boldsymbol{x}_{0:N})- \widehat{\nabla_{\theta}\mathcal{J}(\boldsymbol{x}_{0:N})}\| \leq \frac{\lambda \rho L_{\mathcal{F}}}{1-\lambda} \nonumber
 \end{align}
 \end{subequations}
Therefore, the conclusion is confirmed.
\end{proof}
The two inequalities above provide bounds on the error of the approximate gradient compared to the true gradient. It is evident that these bounds depend on a factor $\lambda$, which governs the convergence rate of the diffusion Picard iteration. As demonstrated in \cite{shih2024parallel, tang2024accelerating, luparasolver}, parallel denoising empirically converges within just a few iterations, which suggests $\lambda$ is small and SDO offers a good approximation of the true gradient. Meanwhile, the gradient shortcuts in SDO mitigate numerical instability that may arise from backpropagation in lengthy diffusion trajectories.

\subsection{Comparison with Existing Methods}

To better understand the contributions of our proposed Shortcut Diffusion Optimization (SDO), we compare it with several existing methods in the literature that address the challenges of backpropagation and optimization in diffusion models. These methods include approaches for efficient gradient computation, diffusion guidance, and gradient truncation.

\textbf{Relation to Existing Backpropagation Methods in Diffusion Models.} Our work is closely related to DOODL \cite{wallace2023end} and AdjointDPM \cite{panadjointdpm}, pioneering studies on backpropagation in diffusion sampling. These methods focus on perfectly inverting the forward process to compute gradients while maintaining the constant memory cost. DOODL achieves this through reversible modules, while AdjointDPM employs the adjoint sensitivity method. In contrast, our approach demonstrates that a perfect reverse process is unnecessary. Instead, a one-step backpropagation suffices, significantly improving computational efficiency compared to these methods. 

\textbf{Comparison with Diffusion Guidance Techniques.} Several methods aim to guide diffusion generation, including UGD \cite{bansal2023universal} and Freedom \cite{yu2023freedom}. While these techniques are effective in specific scenarios, our framework surpasses them in both flexibility and performance. These approaches typically rely on a single-step Tweedie estimation for clean prediction \cite{chung2022diffusion}, which can often result in blurry or unrealistic outputs. This is particularly problematic when downstream metrics $\mathcal{J}(\cdot)$ are provided by another neural network, as it can lead to out-of-distribution issues and unstable gradients \cite{wallace2023end, xu2024consistency}, requiring careful tuning of the guidance strength. Beyond performance, our framework offers greater versatility, supporting a wide range of tasks such as optimizing model parameters, prompt embeddings, and generating adversarial samples, thus making it a more comprehensive solution for diffusion model-based applications.

\textbf{Distinctions from Previous Gradient Truncation Approaches.} The principle of SDO is inspired by recent advancements in differentiating through fixed-point dynamics \cite{fung2022jfb,bolte2024one}, which suggest performing backpropagation only during the final step of the iterative process. We extend this principle to the context of diffusion Picard iteration. Due to the skip dependencies inherent in Picard iteration, gradients for variables at any timestep can still be calculated, even when only the final iteration is retained, effectively creating a gradient shortcut. Some works on tuning diffusion parameters \cite{clark2023directly,li2024physics} have proposed truncating backpropagation to the final discrete ODE timestep. However, it is important to emphasize that their truncation approach fundamentally differs from ours. While their method focuses on truncating gradients at the level of discrete ODE timesteps, our approach operates at the level of Picard iteration steps. A key limitation of their approach is the inability to compute gradients for variables at earlier timesteps. In contrast, our method leverages the structural properties of Picard iteration, offering a principled solution that preserves the ability to compute gradients across all timesteps.

\section{Application}\label{app}
In this section, we apply AdjointDPM to perform several interesting tasks, such as optimizing latent states and model weights for control sampling and specific reward alignment. All experiments were conducted on a single NVIDIA A100 GPU. The implementations
utilized Python 3.8 and PyTorch 1.11 as the primary software stack. The experimental results demonstrate that our method effectively backpropagates metric information to the relevant variables of the generated image. Next, we present the detailed experimental setup and results for each task.

\subsection{Controlled Generation by Latent Optimization}\label{CD}
We begin by optimizing the latent variables within the sampling chain, enabling fine-grained control over individual images.

\subsubsection{Text-Guided Image Manipulation}
\emph{\textbf{Experimental setup.}} Given an image $\boldsymbol{x}_{\text{ref}}$ and a text prompt $\rho$, we consider the following losses:
\begin{equation}
    \mathcal{J}_{\text{text}}=\lambda S_{\text{CLIP}}(\boldsymbol{x}_{0}, \rho)+(1-\lambda)\|\boldsymbol{x}_{0}- \boldsymbol{x}_{\text{ref}}\|
\end{equation}
where the CLIP score $S_{\text{CLIP}}$ optimizes the alignment of the generated image with the prompt \cite{liu2021fusedream}. The second term constrains the optimization to prevent excessive deviation from the original image. The parameter $\lambda$ controls the trade-off between the textual alignment term $S_{\text{CLIP}}$ and the fidelity term that ensures similarity to the original image.  We use the checkpoint of the pixel-space diffusion model provided in \cite{chung2022diffusion}, with sampling performed across 50 steps using DDIM scheduling \cite{song2020denoising}. We conduct experiments on CelebA-HQ \cite{karras2017progressive} and FFHQ \cite{karras2019style}. Optimization is performed over 50 steps using the Adam optimizer \cite{kingma2014adam} with a learning rate of 0.01. The image corresponding to the lowest loss value during the optimization process is recorded as the final result. For attribute manipulation, we set $\lambda = 0.5$, while for other text prompts, we set $\lambda = 0.7$.

\emph{\textbf{Result.}} Fig.~\ref{text_guide} presents qualitative results, demonstrating text-guided edits to facial attributes, appearance, and style. These results indicate that the CLIP loss gradient provided by SDO can effectively guide diffusion models trained solely on facial data to follow diverse text prompts. For quantitative comparisons, we randomly selected 100 samples from each dataset using the following prompts: \{\texttt{old, sad, smiling, angry, curly, Pixar}\}. We compute various metrics to evaluate the performance of each method. The comparison methods include two full backpropagation techniques, AdjointDPM \cite{panadjointdpm} and DOODL \cite{wallace2023end}, as well as two image manipulation baselines: DiffusionCLIP \cite{Kim_2022_CVPR} and FlowGrad \cite{liu2022flow}. The metrics used include the CLIP score, LPIPS \cite{zhang2018unreasonable}, and face ID loss \cite{deng2019arcface}. The CLIP score measures the alignment with the prompt, while the latter two metrics assess the magnitude of changes before and after editing. The quantitative results are presented in Tab.~\ref{tab:1}. We find that SDO significantly outperforms the other two backpropagation algorithms and also provides advantages over the baseline image manipulation methods. Additionally, Fig.~\ref{text_guide_com} shows visual comparisons, where FlowGrad sometimes fails to capture cues such as "\texttt{Angry}", while DiffusionCLIP is over-optimized and does not maintain the original facial features effectively. In contrast, SDO strikes a better balance between the two aspects.

\begin{figure}[!ht]
		\centering 
		\includegraphics[width=\columnwidth]{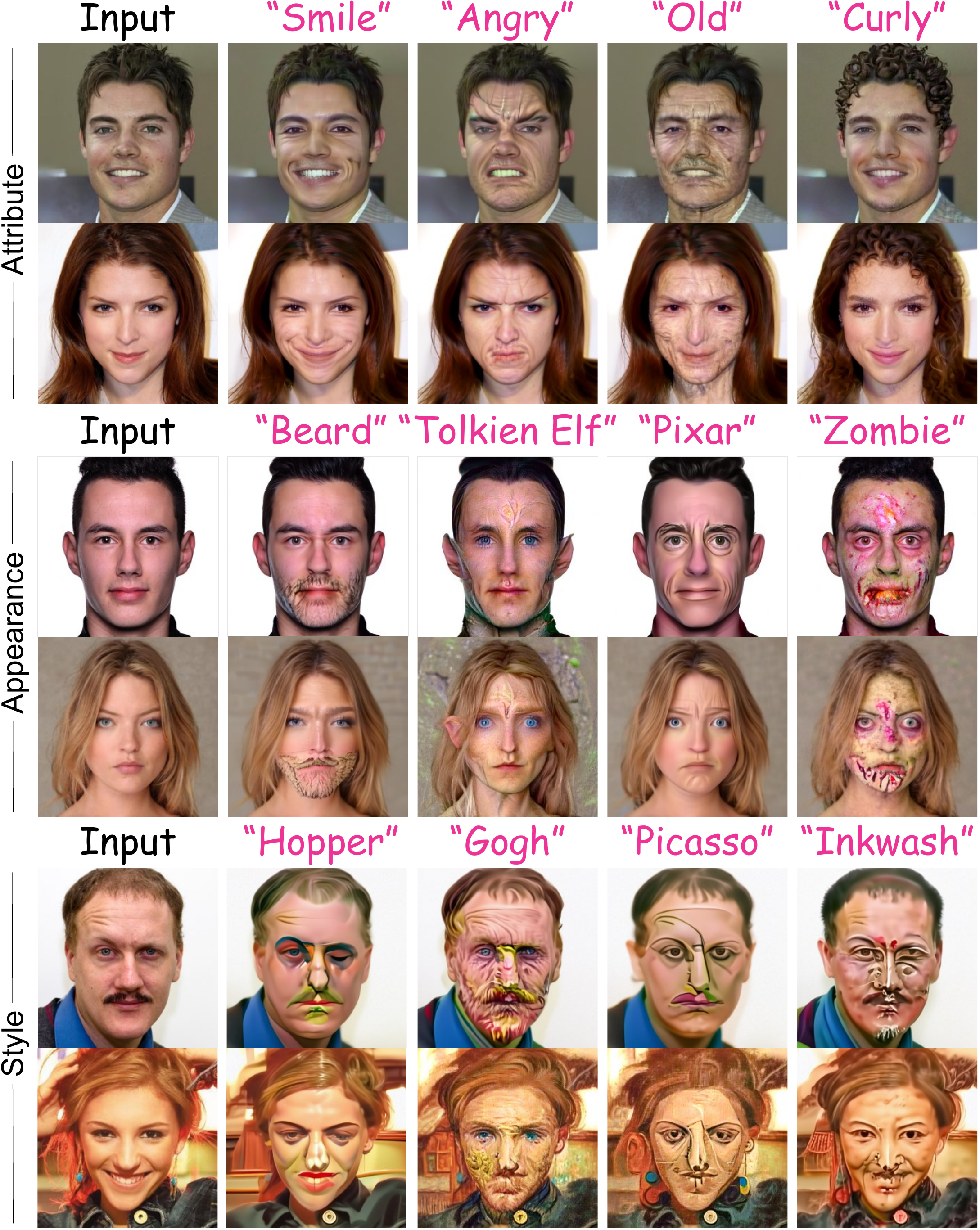}
		\caption{Qualitative results of text-guided image manipulation on samples from CelebA-HQ \cite{karras2017progressive} and FFHQ \cite{karras2019style}. Using text prompts, SDO enables edits to facial attributes, appearance, and style in the images.}  \label{text_guide}
\end{figure}

\begin{figure}[!ht]
		\centering 
		\includegraphics[width=0.9\columnwidth]{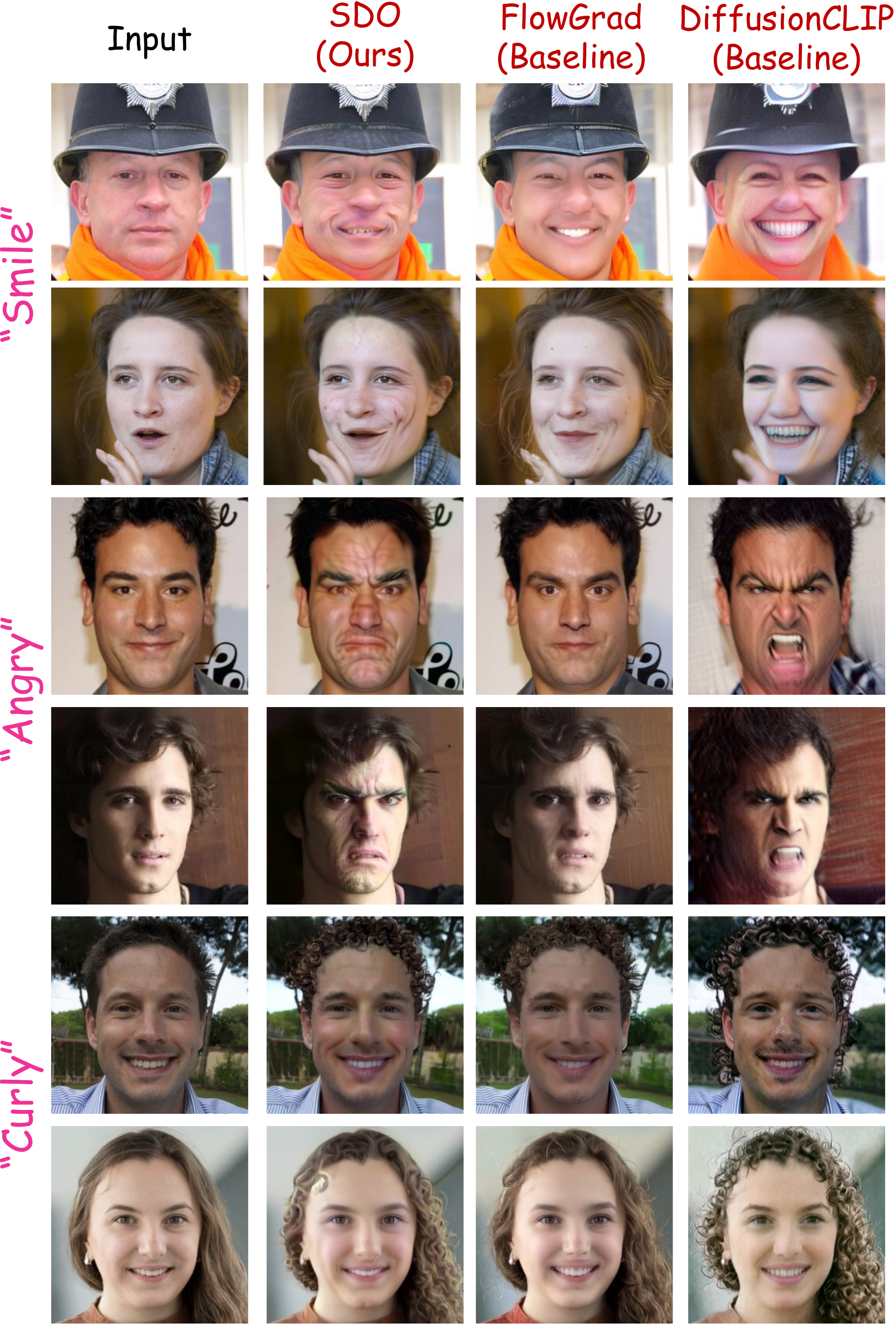}
		\caption{Comparing methods for text-guided image manipulation, including SDO (Ours), DiffusionCLIP \cite{Kim_2022_CVPR} and FlowGrad \cite{liu2023flowgrad}. SDO can follow textual cues well compared to these two and preserve face features.}  \label{text_guide_com}
\end{figure}

\begin{table}[!h]
  \centering
    \caption{Quantitative comparison of text-guided image manipulation between different algorithms.}
    \centering\resizebox{0.44\textwidth}{!}{
      \begin{tabular}{lccc}
    \toprule
    Method & LPIPS$\downarrow$ & CLIP$\uparrow$ & ID loss$\downarrow$\\
    \midrule
    DiffusionCLIP \cite{Kim_2022_CVPR}&0.175&29.93&0.901 \\
    FlowGrad  \cite{liu2023flowgrad}& 0.142&31.30& 0.797 \\
    \midrule
    AdjointDPM \cite{panadjointdpm}& 0.188&28.14&0.908 \\
    DOODL \cite{wallace2023end}&0.183&28.76&0.914\\
    \rowcolor{blue!10}
    SDO (Ours) &\bf 0.134 &\bf 32.18&\bf 0.790\\
    \bottomrule
  \end{tabular}}
    \label{tab:1}
    \end{table}
    \begin{table}[!h]
      \caption{Quantitative comparison of style-guided image generation.}
    \centering\resizebox{0.35\textwidth}{!}{
\begin{tabular}{lcc}
      \toprule
      Method & Style loss$\downarrow$ & CLIP$\uparrow$ \\
      \midrule
      UGD \cite{bansal2023universal}& 791.4& 28.50\\
      FreeDoM \cite{yu2023freedom}& 463.9 & 27.04\\
      \midrule
      AdjointDPM \cite{panadjointdpm}&  396.7& 31.03\\
      DOODL \cite{wallace2023end}& 357.2&\bf 31.12\\
      \rowcolor{blue!10}
      SDO (Ours) &\bf 268.1 & 30.78 \\
      \bottomrule
    \end{tabular}}
    \label{tab:2}
\end{table}

\begin{figure}[!ht]
		\centering 
		\includegraphics[width=\columnwidth]{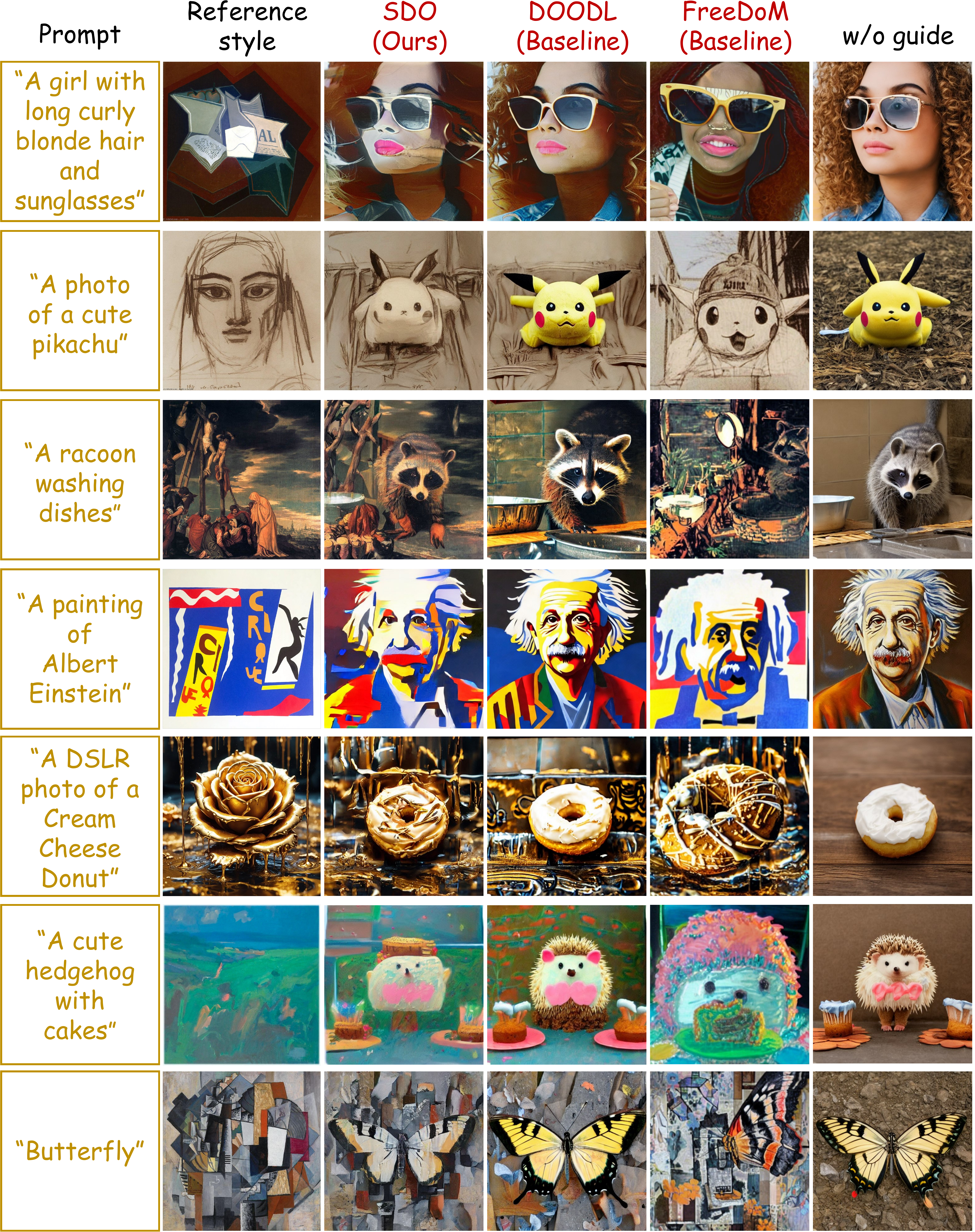}
		\caption{Comparison of style-guided generation between SDO, DOODL \cite{wallace2023end}, and FreeDoM \cite{yu2023freedom}. SDO preserves the semantics of unguided images while offering significant improvements in image quality and stylistic alignment.}  \label{style_guide}
\end{figure}

\subsubsection{Style-Guided Image Generation}
 \emph{\textbf{Experimental setup.}} To customize the style of generated images, we minimize the Gram matrix difference \cite{johnson2016perceptual} between the generated image and the reference image. To quantify the texture style similarity between the two images, we employ the following metric: 
 \begin{equation}
\mathcal{J}_{\text{dist}}=\left\|\mathcal{E}\left(\boldsymbol{x}_{\text {ref}}\right)_j-\mathcal{E}\left(\boldsymbol{x}_{0}\right)_j\right\|_F^2 
\end{equation}
where $\boldsymbol{x}_{\text {ref}}$ represents the reference image for style guidance, and $\mathcal{E}(\cdot)_j$ denotes the Gram matrix \cite{johnson2016perceptual} of the  $j$-th  layer feature map of an image encoder. In our experiments, we select features from the third layer of the CLIP image encoder, which yielded satisfactory results. We optimize the loss using the Adam optimizer for 150 steps, with sampling performed using DDIM for 50 steps. Notably, we observe that optimizing the latent variable at $t = 600$ (out of a total of 1000 steps) better preserved the semantic information of the original image compared to directly optimizing the initial latent. For quantitative comparisons, we use style images from \cite{yu2023freedom}, along with a curated selection from the WikiArt dataset \cite{saleh2015large}, resulting in a total of 15 reference style images. From these references, 30 images were generated using diverse prompts to facilitate both quantitative and qualitative evaluation.
 
  \emph{\textbf{Result.}} Tab.~\ref{tab:2} displays the quantitative results, where SDO significantly outperforms alternatives, including guidance-based approaches such as UGD \cite{bansal2023universal} and FreeDoM \cite{yu2023freedom}. Fig.~\ref{style_guide} presents some visual examples, showing that SDO clearly produces images consistent with the style of the reference image. In contrast, while DOODL performs full backpropagation to calculate precise gradients, it struggles to effectively integrate metric information into the latent variables of the generated image due to the lengthy sampling trajectory. Another unique advantage of our optimization method is that it better preserves the semantic information of the original image, a feature not guaranteed by methods such as FreeDoM \cite{yu2023freedom}.

We separately calculate two diffusion models: one for pixel space \cite{chung2022diffusion} and one for latent space \cite{rombach2022high}. We measure the time and memory required for SDO, Adjoint, and DOODL to compute the gradient of a noisy variable once. Tab.~\ref{tab:time} shows the gradient computation costs. Compared to full backpropagation, SDO achieves approximately a 90\% reduction in computation time and a 35\% reduction in memory usage. This demonstrates the significant efficiency advantage of our method and has great potential to facilitate the practical application of diffusion generation modeling. 


\begin{table}[ht]
  \caption{Comparison of the runtime and memory required to compute the gradient for $\boldsymbol{x}_N$ once using pixel and latent diffusion models on an A100 GPU. The efficiency improvement over baseline backpropagation approaches is highlighted in \textcolor{plainbackprop}{red}.}
    \centering\resizebox{0.5\textwidth}{!}{
    \begin{tabular}{l@{\hskip 7pt}c@{\hskip 7pt}>{\columncolor{blue!10}}c@{\hskip 7pt}c@{\hskip 7pt}c}
        \toprule
         &&SDO (Ours) &AdjointDPM &DOODL \\
                  \midrule
         \multirow{4}{*}{Pixel}&Time (s)& 1.62& 18.73 & 29.96 \\
         &&\textcolor{plainbackprop}{($-$91.35\%)} &- & -\\
         &Memory (GB)& 6.12  & 9.37 & 11.25\\
         &&\textcolor{plainbackprop}{($-$34.68\%)} & - &-\\
         \midrule
         \multirow{4}{*}{Latent}&Time (s)& 1.91 & 17.12 & 27.20 \\
         &&\textcolor{plainbackprop}{($-$88.84\%)} & -& -\\
         &Memory (GB)& 10.33  & 15.88 & 18.59\\
         &&\textcolor{plainbackprop}{($-$34.95\%)} & - & -\\
         \bottomrule
    \end{tabular}}
    \label{tab:time}
\end{table}

\begin{table}[!h]
  \centering
      \caption{Quantitative comparison of aesthetic enhancement.}
    \centering\resizebox{0.32\textwidth}{!}{
    \begin{tabular}{lc}
      \toprule
      Method & Aesthetic score$\uparrow$ \\
      \midrule
      SD v1.4 \cite{rombach2022high}&6.27 \\
       UGD \cite{bansal2023universal}& 8.13\\
      FreeDoM \cite{yu2023freedom}& 8.74 \\
      \midrule
      AdjointDPM \cite{panadjointdpm}& 8.96 \\
      DOODL \cite{wallace2023end}&9.02 \\
      \rowcolor{blue!10}
      SDO (Ours) &\bf 9.13 \\
      \bottomrule
    \end{tabular}}
    \label{tab:3}
    \end{table}
\subsubsection{Aesthetic Enhancement}
\emph{\textbf{Experimental setup.}} We now focus on enhancing the aesthetics of generated images using an aesthetic predictor \cite{pressman2022simulacra} that outputs a score between 1 and 10. To improve the aesthetic score of an image, we employ the following loss function:
\begin{equation}
\mathcal{J}_{\text{aes}}= \left|a\left(\boldsymbol{x}_0\right)-10\right| 
\end{equation}
where $a(\boldsymbol{x}_0)$ is the aesthetic predictor, which outputs values within the range of $[1, 10]$. The goal is to maximize the score, targeting a value of 10. Similar to style-guided generation, we optimize the latent variables at $t=600$ to better preserve the semantic information of the original image. Additionally, aesthetic enhancement can be applied to real-world images by encoding them into latent variables via DDIM inversion \cite{song2020denoising}. For real-world images, we use DDIM inversion to encode the image into the latent space at $t=600$ before performing the aesthetic enhancement. Optimization is conducted over 100 steps using the Adam optimizer with a learning rate of 0.01. For quantitative evaluation, we employ two types of prompts: \texttt{painting} and \texttt{photo}. We randomly select 20 prompts from each category, and the optimized aesthetic scores of the generated images are computed for comparison.

\emph{\textbf{Result.}} We compare the resulting aesthetic scores of all generated images with baseline models such as Stable Diffusion (SD) v1.4, UGD \cite{bansal2023universal}, FreeDoM \cite{yu2023freedom}, AdjointDPM \cite{panadjointdpm}, and DOODL \cite{wallace2023end} in Tab.~\ref{tab:3}. As shown in the table, SDO proves to be more effective in guiding the diffusion model to generate images with higher aesthetic scores. This demonstrates the effectiveness of SDO in disseminating gradient information. Additionally, Fig.~\ref{aes_guide} provides visual examples before and after optimization. We find that this method not only improves the aesthetic scores of the generated images but also successfully works on real-world images, producing satisfactory results.

\begin{figure}[!ht]
		\centering 
		\includegraphics[width=\columnwidth]{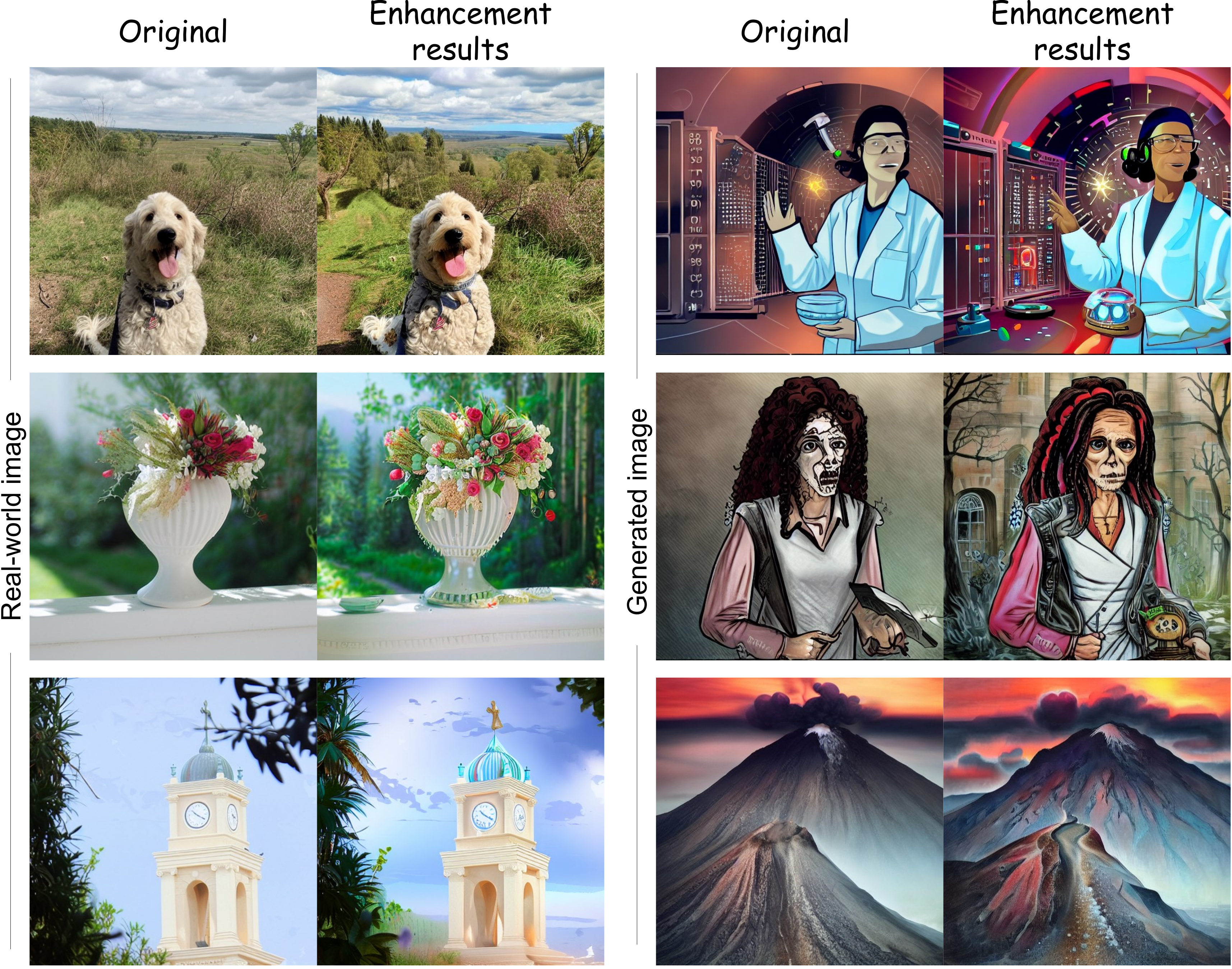}
		\caption{Qualitative results of aesthetic enhancement applied to both real-world images (left) and generated images (right).}  \label{aes_guide}
\end{figure}

\begin{figure}[!ht]
		\centering 
		\includegraphics[width=\columnwidth]{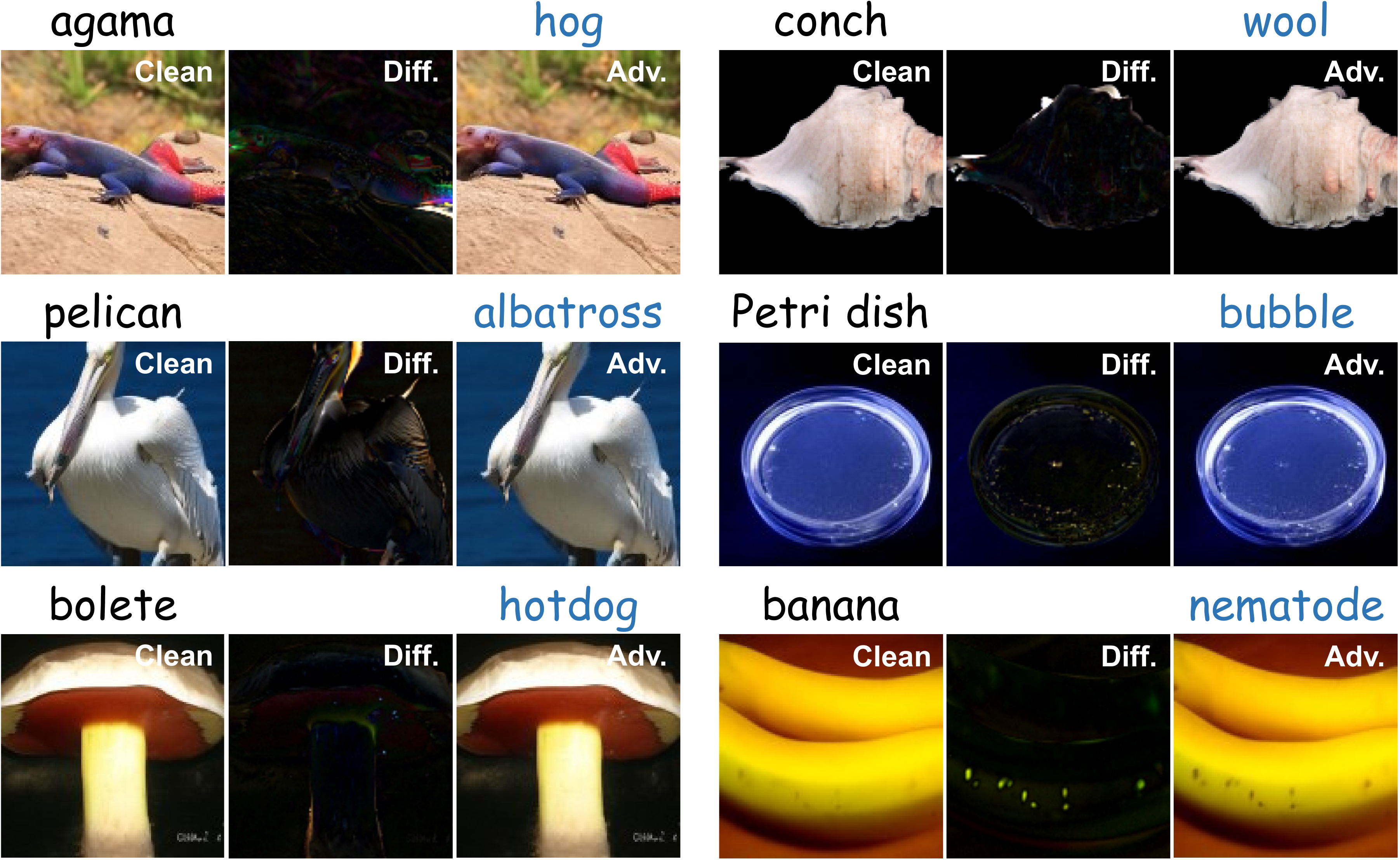}
		\caption{Adversarial examples targeting the ResNet50 \cite{he2016deep} classifier trained on ImageNet \cite{deng2009imagenet}. From left to right, each set shows a clean sample, the amplified ($\times 5$) difference before and after adversarial perturbation, and the resulting adversarial sample.}  \label{attack_guide}
\end{figure}

\begin{figure*}[!t]
  \centering
  \begin{subfigure}[b]{0.335\textwidth}
\includegraphics[width=\textwidth]{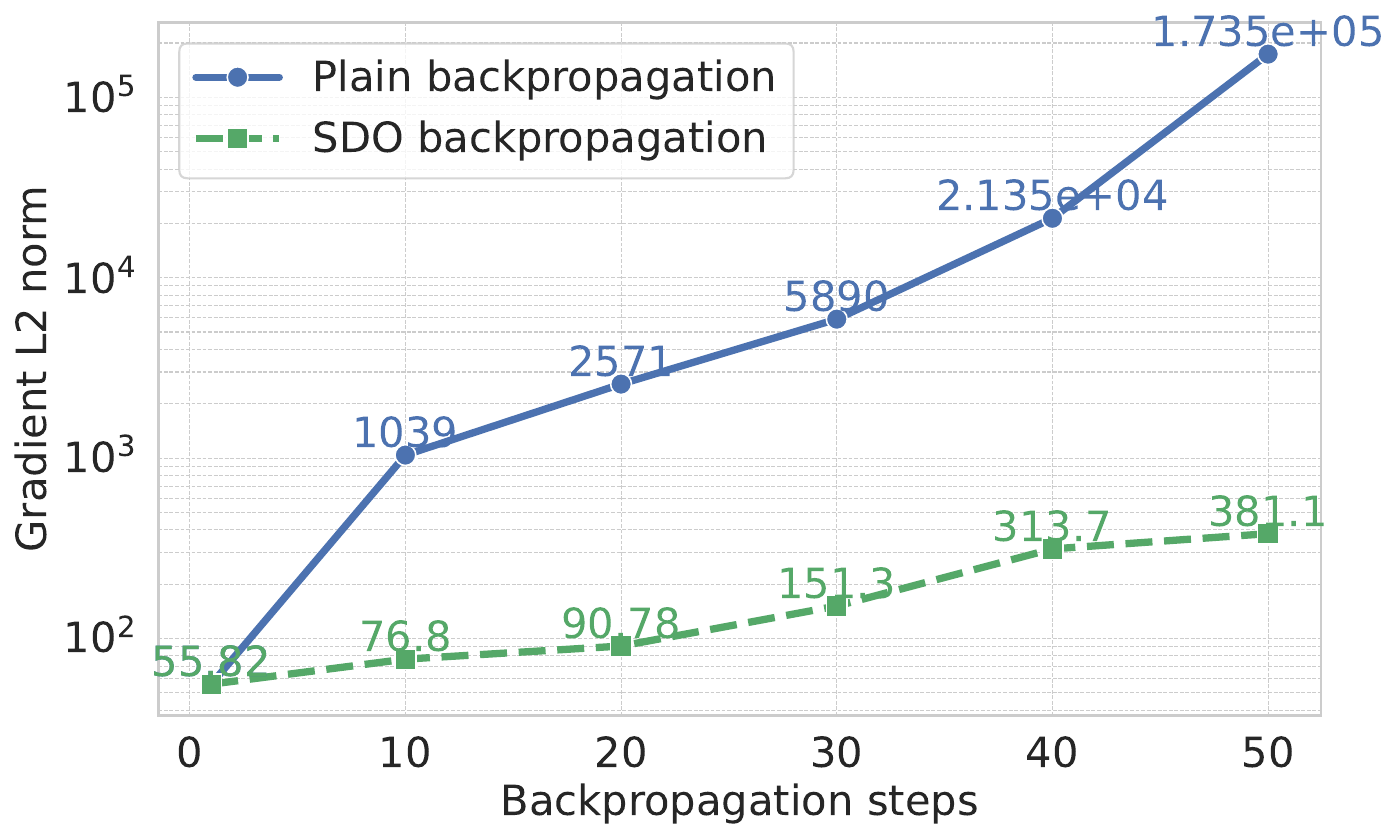}
    \caption{Gradient L2 norms.}
    \label{fig:sub1}
  \end{subfigure}
  \hfill 
  \begin{subfigure}[b]{0.315\textwidth}
\includegraphics[width=\textwidth]{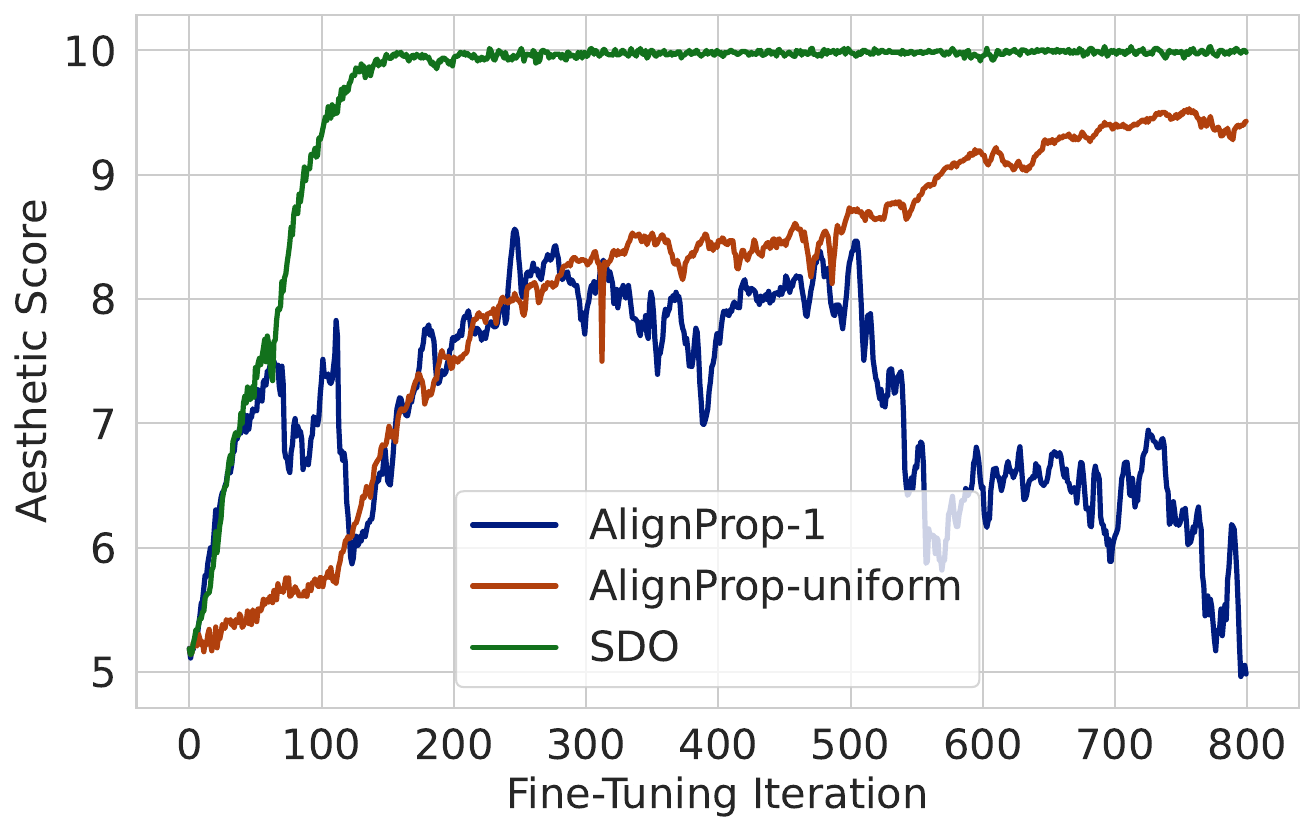}
    \caption{Aesthetic score.}
    \label{fig:sub2}
  \end{subfigure}
    \hfill 
  \begin{subfigure}[b]{0.325\textwidth}
    \includegraphics[width=\textwidth]{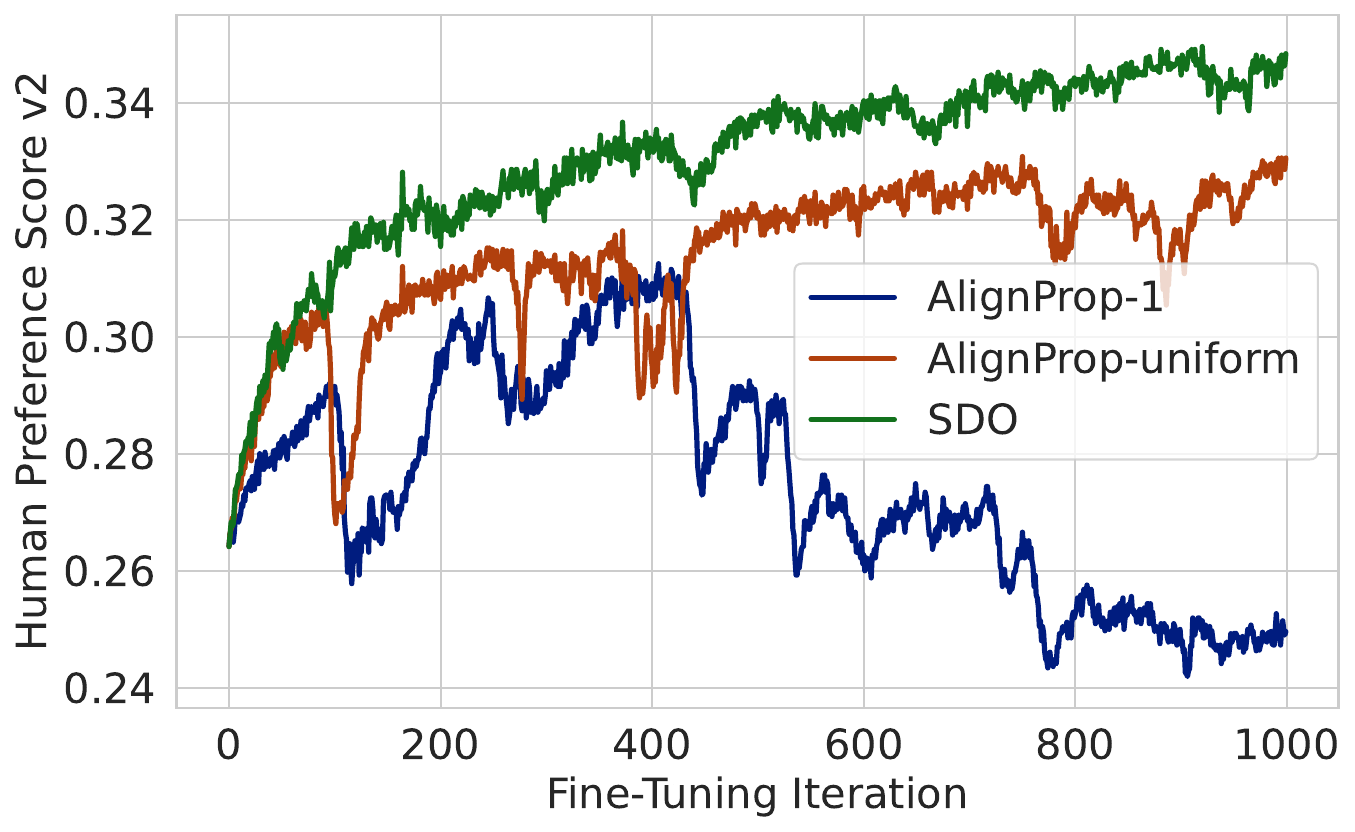}
    \caption{Human Preference Score v2.}
    \label{fig:sub3}
  \end{subfigure}
\caption{(a) Gradient norms of parameters in plain backpropagation explode as steps increase, potentially causing optimization challenges. (b) and (c) illustrate the changes in Aesthetic and HPSv2 scores during the fine-tuning process, where these metrics serve as reward functions. SDO demonstrates greater empirical stability and faster convergence.}
  \label{fig:curve}
\end{figure*}

\subsubsection{Adversarial Sample Generation}
 \emph{\textbf{Experimental setup.}}  Prior work \cite{panadjointdpm} has shown that existing pretrained diffusion models (DMs) can produce content capable of bypassing audit filters, such as $f_{\text{cls}}$. To generate adversarial samples, we first use conditional generation to produce samples from a specific class. Subsequently, the noisy latent corresponding to the generated samples is optimized to misclassify the sample. This can be formulated as an optimization problem over the adversarial initial noise:
\begin{equation}
    \max_{\|\delta\|_{\infty}<\tau} \mathcal{J}_{\text{dis}}\left(c, f_{\text{cls}}\left(G\left(\boldsymbol{x}_N+\delta, \epsilon_\theta,c\right)\right)\right) 
\end{equation}
where $c$ represents the harmful condition, and $\mathcal{J}_{\text{dis}}$ measures the distance between the generated content and harmful signals. The latent variables are optimized to maximize this distance, ensuring that the adversarial samples bypass the classifier. To preserve the condition information of the original image, the perturbation of the latent variable is constrained by $\|\delta\|_{\infty} < \tau$. For quantitative comparisons, we use the models pre-trained on ImageNet $128\times128$ \cite{dhariwal2021diffusion} and employed classifier guidance to generate 100 samples for each class. A total of 13 random classes were selected for optimization, with the goal of bypassing a trained ResNet50 \cite{he2016deep} classifier. After excluding samples that were already misclassified by the classifier, 1,179 samples remained for the experiment. Sampling was performed over 30 discrete steps, and optimization was conducted over 30 steps using gradient descent with a learning rate of 0.1. After each optimization step, the perturbation $\delta$ was clipped to the range $[-0.1, 0.1]$. 

Following optimization, we measure the success rate of bypassing the classifier and compute the LPIPS metric to assess perceptual similarity with the original image. Additionally, we calculate the Fréchet Inception Distance (FID) \cite{heusel2017gans} of adversarial samples to measure image quality. FID was computed between the generated images and the ImageNet validation set.

\begin{figure*}[!t]
		\centering 
		\includegraphics[width=\textwidth]{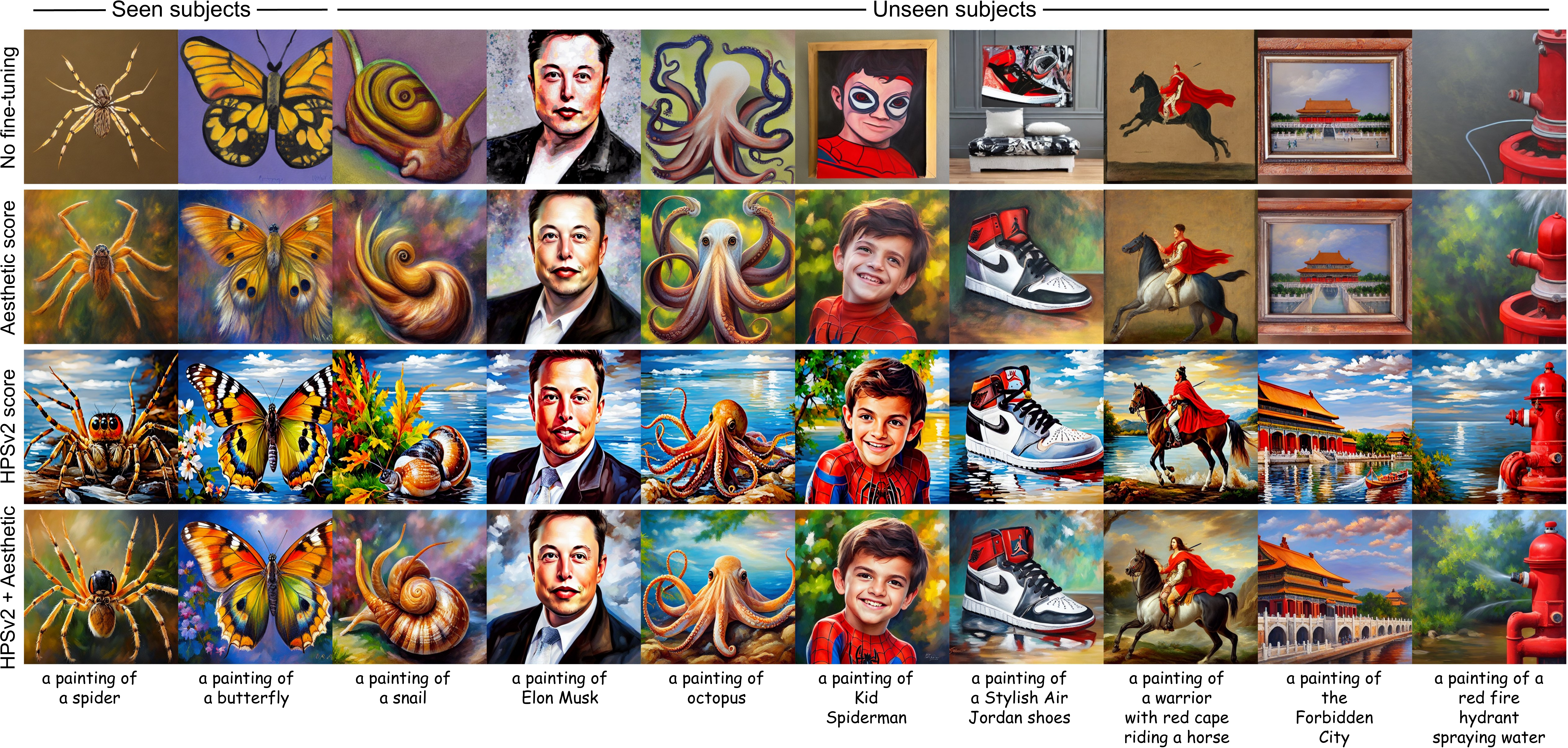}
\caption{\textbf{Comparison of generated images from fine-tuned models using the same seeds.} Each row, from top to bottom, displays results for: no fine-tuning, fine-tuning on Aesthetic score, fine-tuning on HPSv2, and a mixture of both LoRA weights. The model fine-tuned on a simple prompt set also demonstrates the ability to generalize to previously unseen subjects.} \label{reward_fine}
\end{figure*}

 \emph{\textbf{Result.}} As shown in Tab.~\ref{tab:4}, our SDO method outperforms the baseline on this task. Specifically, most of the SDO-optimized samples successfully bypass the classifier. In contrast, the full backpropagation methods achieve an attack success rate of approximately 50\%. This demonstrates the struggle of full backpropagation methods, as they fail to effectively propagate the information from the classifier to the latent variables. More importantly, we find that SDO optimization exhibits significantly less changes (indicated by a lower LPIPS metric) and is accompanied by less degradation in image quality (reflected in a lower FID metric). Fig.~\ref{attack_guide} shows examples of the generated adversarial samples and the differences before and after optimization. Although these images show only minor changes, they successfully achieve adversarial goals. These results raise concerns about the potential for generative models to produce harmful samples, underscoring the need for the development of more advanced protective tools in the future \cite{qu2023unsafe}.

    \begin{table}[!h]
        \caption{Average success rate (\%) of the adversarial initial state for bypassing the classifier, along with FID and LPIPS values before and after optimization.}
 \centering\resizebox{0.45\textwidth}{!}{
    \begin{tabular}{lccc}
      \toprule
      Method & Success ratio$\uparrow$ & LPIPS$\downarrow$& FID$\downarrow$\\
      \midrule
      AdjointDPM \cite{panadjointdpm}& 52.54&0.098& 85.3 \\
      DOODL \cite{wallace2023end}&51.61&0.084& 70.4\\
      \rowcolor{blue!10}
      SDO (Ours) &\bf 96.32 &\bf 0.023&\bf 58.2\\
      \bottomrule
    \end{tabular}}
    \label{tab:4}
\end{table}

\subsection{Fine-tuning for Reward Alignment}\label{RA}

Beyond controlling individual images, we can also fine-tune the parameters of the DM so that it does not need to be optimized again when generating new images. Here, we address the problem of aligning a pre-trained text-to-image DM with a specific reward to improve output quality.

\emph{\textbf{Experimental setup.}} Following the approach in AlignProp \cite{prabhudesai2023aligning}, we perform end-to-end fine-tuning of a differentiable reward by directly backpropagating along the sampling chain in Stable Diffusion (SD). The objective of reward fine-tuning is mathematically expressed as: 
\begin{equation}
\mathcal{J}_{\text{reward}}=\mathbb{E}_{\rho , \mathbf{x}_T \sim \mathcal{N}(\mathbf{0}, \mathbf{I})}\left[r\left(G\left(\boldsymbol{x}_N, \epsilon_\theta, \rho \right), \rho\right)\right] 
\end{equation}
where prompt $\rho$ is sampled from a predefined prompt set. 

We consider two types of rewards: aesthetic score \cite{pressman2022simulacra} and Human Preference Score v2 (HPSv2) \cite{wu2023human}. Instead of fine-tuning the weights of the original DM, we adopt the approach from \cite{prabhudesai2023aligning, clark2023directly} by adding low-rank adapter (LoRA) \cite{hu2021lora} modules for fine-tuning. To assess the efficiency of the gradient shortcut provided by SDO, we compare it against the following baselines: \textbf{AlignProp-uniform}: As noted in \cite{prabhudesai2023aligning}, full backpropagation can lead to mode collapse, so this approach applies a random truncated backpropagation strategy with uniformly selected truncated timesteps. \textbf{AlignProp-1}: This method only performs backpropagation at the last timestep to alleviate the computation cost, as used in \cite{clark2023directly} and \cite{li2024physics}.

We compare the baseline methods with our proposed SDO algorithm under a similar experimental setup. Specifically, we use Stable Diffusion v1.5 as the base diffusion model. Images were generated using a 50-step DDIM scheduler and classifier-free guidance with a weight of 7.5. For training, we used the AdamW optimizer \cite{loshchilov2017decoupled} with $\beta_1 = 0.9$, $\beta_2 = 0.999$, and a weight decay of $1e-4$. The learning rate was uniformly set to $4e-4$. To minimize memory usage during training, the pre-trained Stable Diffusion parameters were converted to bfloat16, while the LoRA parameters were retained in float32 precision. Notably, we apply LoRA only to the parameters of the score network, excluding the text encoder and VAE decoder. The batch size was set to 16 to fit within our computational resources, and gradient accumulation was performed over four iterations. For both reward functions, we use a prompt set containing 45 animals \cite{black2023training} with the format \emph{"a painting of a(n) [animal]"} and performed 1,000 steps of fine-tuning.

\emph{\textbf{Result.}} We first investigate the relationship between backpropagation depth and parameters in the sampling trajectory. The results are shown in Fig.~\ref{fig:sub1}, where we observe that the norm of the gradient generated by naive backpropagation increases sharply with the number of backward steps, leading to gradient explosion, which can cause optimization instability. In contrast, the gradient produced by SDO remains stable, regardless of the depth of backpropagation.

Next, we examine the fine-tuning process for reward alignment across different algorithms. Figs.~\ref{fig:sub2} and \ref{fig:sub3} show the reward curves during the fine-tuning process. We observe that AlignProp-1 becomes unstable after a certain number of optimization steps, struggling to generalize effectively to earlier timesteps, which impacts the semantics of the generated images. Additionally, we find that SDO requires only about 50\% of the runtime compared to AlignProp-uniform. Despite using gradient clipping to alleviate potential instability, as shown in Fig.~\ref{fig:sub1}, SDO converges faster than AlignProp-uniform, even with the same number of fine-tuning steps. We hypothesize that this is due to SDO's ability to mitigate gradient explosion through the gradient shortcut.

\begin{figure}[!ht]
		\centering 
		\includegraphics[width=\columnwidth]{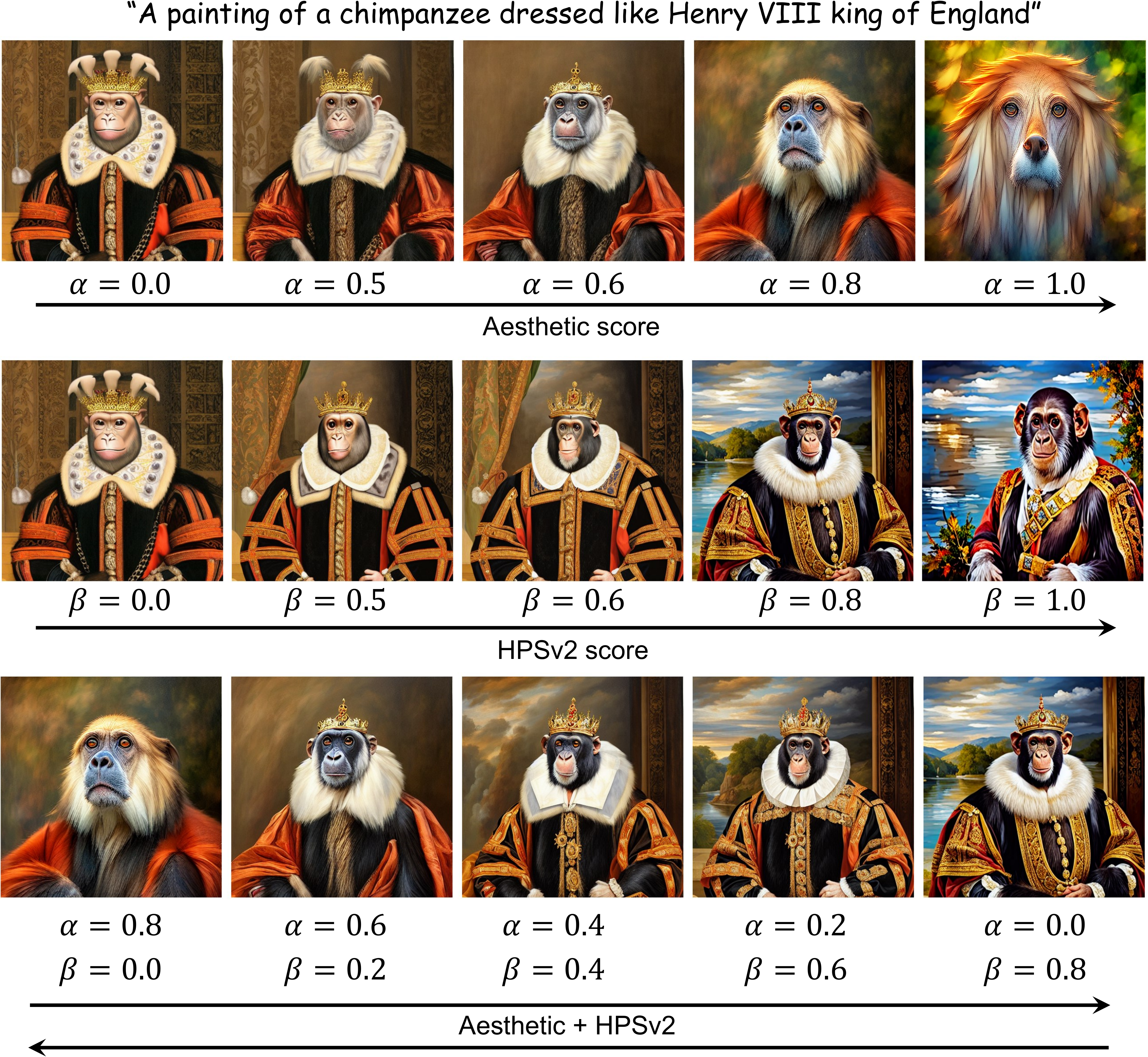}
		\caption{Generated images from SDO fine-tuned models with varying scalar multipliers applied to the LoRA weights (first and second rows). The third row demonstrates the combination of LoRA parameters independently tuned for two distinct rewards.}  \label{lora_inter}
\end{figure}

Although fine-tuning was conducted using a relatively simple set of prompts, Fig.~\ref{reward_fine} shows the results after reward fine-tuning, demonstrating that the model can generalize its style to previously unseen subjects. Moreover, by combining multiple LoRA modules, we achieve a comprehensive effect that integrates multiple independently tuned styles.

In the reward fine-tuning experiments, we also observe instances of reward hacking \cite{prabhudesai2023aligning, clark2023directly}, where over-optimization led to a loss of diversity in the generated outputs. It is particularly noticeable in tasks focused on maximizing aesthetic scores, where the model tended to produce highly specific, high-reward images at the expense of variety. However, as shown in Fig.~\ref{lora_inter}, this issue can be mitigated by controlling the amplitude of LoRA adjustments through simple interpolation ($\theta_{\text{old}} + \alpha \theta_{\text{Aesthetic}}$). Furthermore, independently fine-tuned LoRAs can be fused to achieve a more balanced effect ($\theta_{\text{old}} + \alpha \theta_{\text{Aesthetic}} + \beta \theta_{\text{HPSv2}}$), enabling more comprehensive reward alignment. These results demonstrate the effectiveness of SDO fine-tuning in generating diverse and high-quality outputs.

\subsection{Application to other diffusion frameworks}

In this section, we further explore the general applicability of our approach by testing it within different diffusion modeling frameworks. One popular framework is the Diffusion Transformer (DIT) \cite{peebles2023scalable,ma2024sit}, which replaces the U-Net architecture with a transformer backbone. This architecture is widely used for generative modeling due to its scalability. Since SDO makes no assumptions about the underlying network architecture, it is well-suited for transformer-based models. To validate this, we conduct experiments using the checkpoints provided by PixArt-$\alpha$ \cite{chen2023pixart}, which consists entirely of transformer blocks for latent diffusion.

\begin{figure}[!ht]
		\centering 
		\includegraphics[width=\columnwidth]{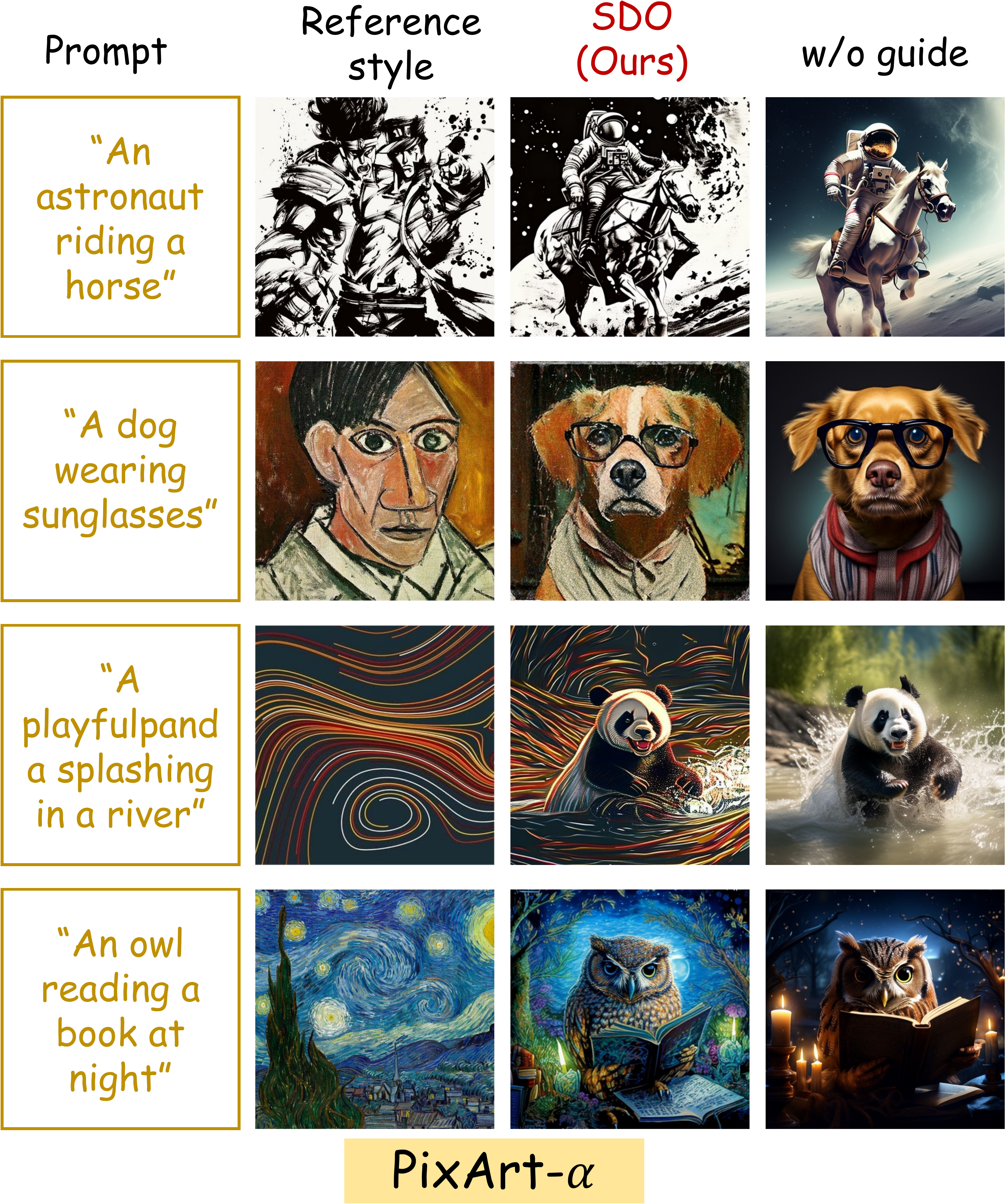}
		\caption{Style-guided generation using PixArt-$\alpha$ model \cite{chen2023pixart} with DIT backbone.}  \label{pixart}
\end{figure}

Fig.~\ref{pixart} showcases some of the results generated by the style guide. We observe that SDO can be successfully applied to advanced models like DIT, enhancing controllability and enriching downstream applications.

Additionally, we consider another important framework: the Consistency Model (CM) \cite{song2023consistency}. CM accelerates the diffusion generation process by enabling one-step generation, but this comes with a tradeoff in quality. To improve output quality, CM increases the number of sampling steps, known as multistep consistency sampling. Our experiments demonstrate that SDO can also be integrated with CM. Specifically, we keep a single-step computation graph of multistep consistency sampling, similar to that in Fig.~\ref{sample-figure}. We test this approach using Latent Consistency Models (LCMs) \cite{luo2023latent}, and the results, presented in Fig.~\ref{lcm}, further illustrate the potential of SDO in enhancing a wide variety of diffusion modeling approaches.

\begin{figure}[!ht]
		\centering 
		\includegraphics[width=\columnwidth]{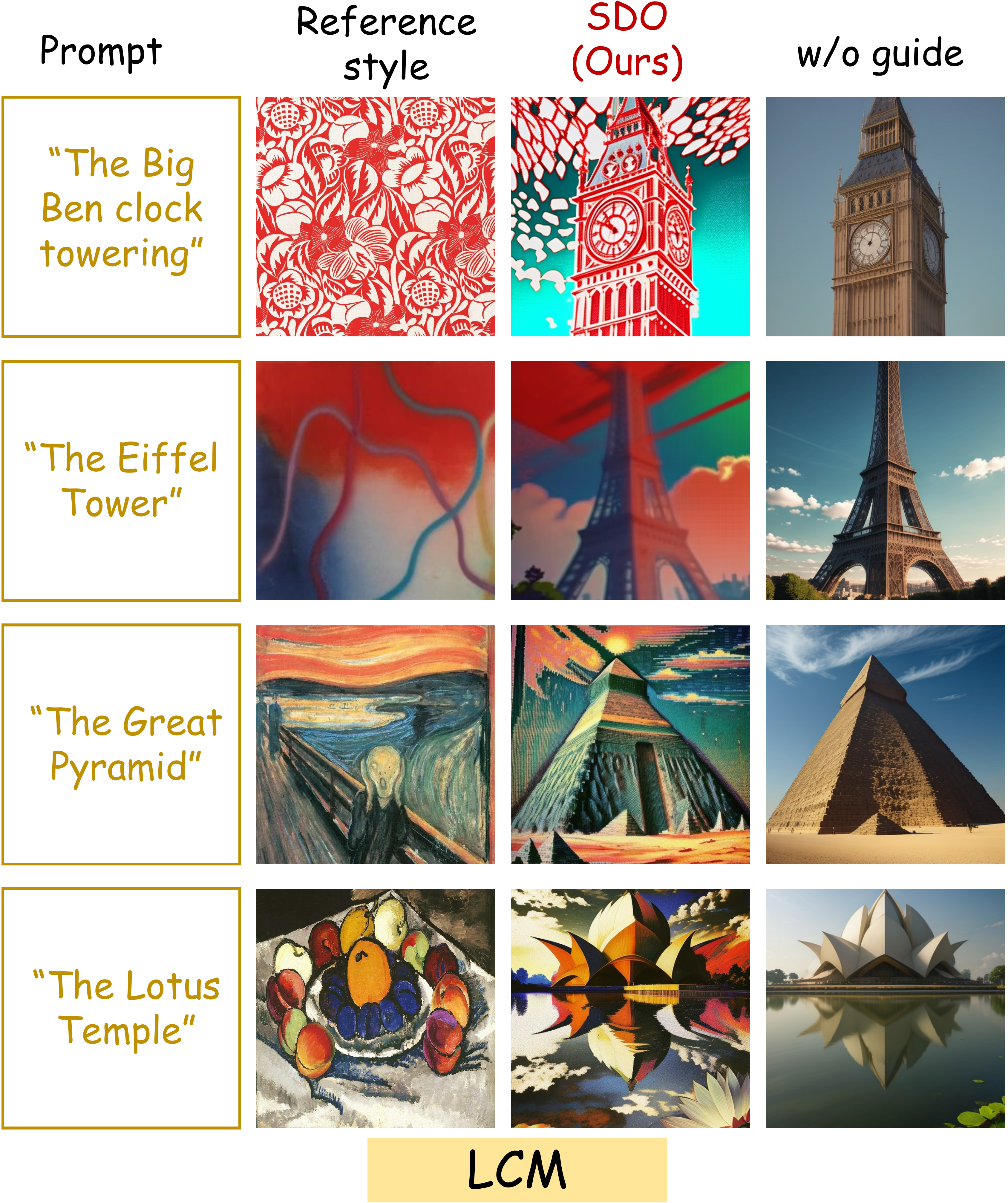}
		\caption{Style-guided generation using latent consistency model (LCM) \cite{luo2023latent}.}  \label{lcm}
\end{figure}

\subsection{Application to other diffusion solvers}

So far, we have primarily used the widely adopted DDIM solvers. However, to enhance sampling efficiency, more advanced diffusion solvers have been proposed. A notable example is the DPM-solver \cite{lu2022dpm}, which utilizes the semilinear structure of diffusion ODEs and incorporates higher-order information. We aim to verify the applicability of SDO to such solvers.

Our experiments show that SDO is indeed compatible with the DPM-solver. The only modification required is to replace the \texttt{Scheduler} in the algorithm with the \texttt{DPMSolverMultistepScheduler} from the Diffuser \cite{von-platen-etal-2022-diffusers} library (Fig.~\ref{SDO_x}), while maintaining a single step of the computational graph. As shown in Fig.~\ref{dpm}, we perform text-guided image manipulation experiments using the second-order DPM solver, and the results are comparable to those obtained with the DDIM solver.

\begin{figure}[!h]
		\centering 
		\includegraphics[width=0.9\columnwidth]{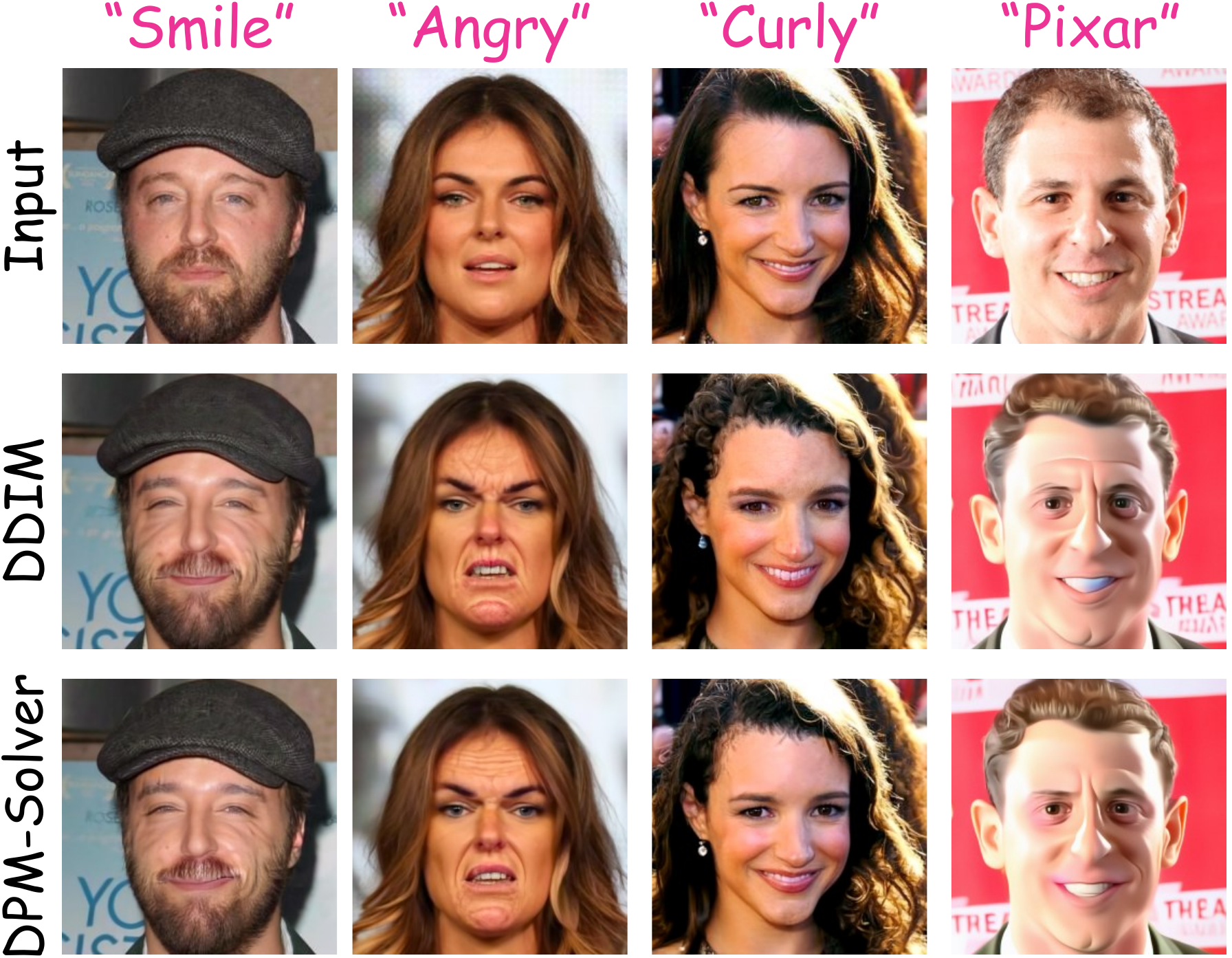}
		\caption{Visualization of text-guided face manipulation using DDIM \cite{song2020denoising} and DPM-Solver \cite{lu2022dpm}.}  \label{dpm}
\end{figure}

\section{Conclusions and limitations}\label{conc}

In this paper, we introduce Shortcut Diffusion Optimization (SDO), a stable and efficient alternative to full backpropagation for diffusion sampling. Inspired by parallel sampling via Picard iteration, SDO computes gradients at only a single step in the computation graph, significantly reducing computational overhead. This approach enables the effective optimization of all parameters along the sampling trajectory for any differentiable objective.

We validate SDO across a variety of tasks, demonstrating its versatility and computational advantages in scenarios such as controllable generation and reward-based alignment. By eliminating the need for full backpropagation through time, SDO represents a practical and scalable solution for training diffusion-based generative models.

Despite its benefits, SDO has certain limitations. Notably, it requires the downstream supervisory signal to be differentiable, which may restrict its applicability in settings involving non-differentiable objectives. Moreover, in reward optimization tasks, we observe reward hacking behaviors consistent with prior findings \cite{clark2023directly,prabhudesai2023aligning}, where the model over-optimizes for the reward signal in unintended ways. While a full investigation of this issue is beyond our current scope, it suggests a promising direction for future work. In particular, integrating regularization strategies \cite{deng2024prdp} may help address these challenges and further improve the robustness of diffusion optimization.

\bibliographystyle{IEEEtran}
\bibliography{IEEEabrv,example_paper}

\end{document}